\pdfoutput=1

\documentclass[11pt]{article}

\usepackage[preprint]{acl}

\usepackage{times}
\usepackage{latexsym}

\usepackage[T1]{fontenc}

\usepackage[utf8]{inputenc}

\usepackage{microtype}

\usepackage{inconsolata}

\usepackage{graphicx}

\usepackage{subcaption}
\usepackage{wrapfig}

\usepackage{amsmath,amsfonts,bm}
\usepackage{booktabs} 
\usepackage{amsthm}
\usepackage{dsfont}

\newtheorem{theorem}{Theorem}
\newtheorem{lemma}{Lemma}

\newtheorem{definition}{Definition}
\newtheorem{assumption}{Assumption}
\newtheorem{remark}{Remark}

\def\eqref#1{equation~\ref{#1}}

\def\1{\bm{1}}

\def\vzero{{\bm{0}}}
\def\vone{{\bm{1}}}

\def\va{{\bm{a}}}
\def\vb{{\bm{b}}}
\def\vc{{\bm{c}}}
\def\vd{{\bm{d}}}
\def\ve{{\bm{e}}}
\def\vf{{\bm{f}}}
\def\vg{{\bm{g}}}
\def\vh{{\bm{h}}}

\def\vk{{\bm{k}}}

\def\vn{{\bm{n}}}
\def\vo{{\bm{o}}}
\def\vp{{\bm{p}}}
\def\vq{{\bm{q}}}

\def\vs{{\bm{s}}}

\def\vv{{\bm{v}}}

\def\vx{{\bm{x}}}
\def\vy{{\bm{y}}}
\def\vz{{\bm{z}}}

\def\mA{{\bm{A}}}

\def\mH{{\bm{H}}}
\def\mI{{\bm{I}}}

\def\mO{{\bm{O}}}

\def\mW{{\bm{W}}}
\def\mX{{\bm{X}}}
\def\mY{{\bm{Y}}}
\def\mZ{{\bm{Z}}}

\DeclareMathAlphabet{\mathsfit}{\encodingdefault}{\sfdefault}{m}{sl}
\SetMathAlphabet{\mathsfit}{bold}{\encodingdefault}{\sfdefault}{bx}{n}

\def\gA{{\mathcal{A}}}

\def\gI{{\mathcal{I}}}

\def\gS{{\mathcal{S}}}
\def\gT{{\mathcal{T}}}

\def\gV{{\mathcal{V}}}

\def\sI{{\mathbb{I}}}

\def\sN{{\mathbb{N}}}

\def\sR{{\mathbb{R}}}

\newcommand{\bA}{\overline{\mA}}
\newcommand{\pA}{\widetilde{\mA}}
\newcommand{\pa}{\tilde{\va}}
\newcommand{\bb}{\overline{\vb}}
\newcommand{\vDelta}{\mathbf{\Delta}}

\newcommand{\silu}{\mathrm{SiLU}}

\newcommand{\relu}{\mathrm{ReLU}}

\title{Exploring the Limitations of Mamba in COPY and CoT Reasoning}

 \author{Ruifeng Ren \ \ \ \ \ \ \ Zhicong Li  \ \ \ \ \ \ \ Yong Liu\thanks{Corresponding author} \\
         Gaoling School of Artificial Intelligence \\ Renmin University of China \\ Beijing, China \\
         \texttt{\{renruifeng920, zhicongli, liuyonggsai\}@ruc.edu.cn}
         }

\begin{document}
\maketitle
\begin{abstract}
Transformers have become the backbone of modern Large Language Models (LLMs); however, their inference overhead grows linearly with the sequence length, posing challenges for modeling long sequences. In light of this, Mamba has attracted attention for maintaining a constant inference size, with empirical evidence demonstrating that it can match Transformer performance in sequence modeling while significantly reducing computational costs. However, an open question remains: \textit{can Mamba always bring savings while achieving performance comparable to Transformers?} In this paper, we focus on analyzing the expressive ability of Mamba to perform our defined COPY operation and Chain of Thought (CoT) reasoning. First, inspired by the connection between Mamba and linear attention, we show that constant-sized Mamba may struggle to perform COPY operations while Transformers can handle them more easily. However, when the size of Mamba grows linearly with the input sequence length, it can accurately perform COPY, but in this case, Mamba no longer provides overhead savings. Based on this observation, we further analyze Mamba's ability to tackle CoT tasks, which can be described by the Dynamic Programming (DP) problems. Our findings suggest that to solve arbitrary DP problems, the total cost of Mamba is still comparable to standard Transformers. However, similar to efficient Transformers, when facing DP problems with favorable properties such as locality, Mamba can provide savings in overhead. Our experiments on the copy and CoT tasks further demonstrate Mamba’s limitations compared to Transformers in learning these tasks.
\end{abstract}

\section{Introduction}

Reccently, Transformer-based large language models (LLMs) have become the mainstream of modern neural network architectures due to their outstanding performance across a wide range of tasks \cite{transformer,bert,GPT2,vit,graphtransformer_survey}. 
However, the core component of Transformers—the attention layer—while providing excellent performance, also leads to emerging drawbacks: during training, the computational cost scales quadratically with sequence length, and during inference, the cost scales linearly with sequence length. This limitation becomes increasingly unacceptable when dealing with long sequence tasks. To address this issue, many works have attempted to improve the attention mechanism to reduce its time and memory costs \cite{efficient_transformers_survey,performer,elu+1,longformer,sparse_transformer}. However, these improved structures often achieve efficiency in the attention layer at the expense of some performance.

Faced with the scaling challenges of Transformers, the exploration of new model architectures to replace Transformers has gradually come into focus, leading to the development of modern RNN architectures, including RWKV~\cite{RWKV}, RetNet~\cite{RetNet}, and Mamba~\cite{Mamba}. Among them, the Mamba architecture \cite{SSM,Mamba}, based on the state space model~(SSM), has garnered attention for its performance comparable to Transformers in many sequence modeling tasks \cite{Mamba2} and vision tasks \cite{Mamba_vision,Mamba_vision_survey}. These models utilize hardware-aware algorithms during training, resulting in computational costs that scale linearly with sequence length, and require constant-level computation and memory during inference at each step. Mamba's strong performance and computational efficiency make it a strong competitor to Transformers.

Despite Mamba demonstrating excellent performance, one can not help but ask: Can Mamba always enjoy such "free lunch", that is,  can Mamba always bring overhead savings while solving tasks effectively?
More recent results have revealed Mamba's shortcomings in certain tasks, especially those involving model's retrieval ability~\cite{MQAR,MMLU,phonebook}.
Specifically, \citet{icl_different_models} study the in-context language learning capabilities of different models and find that Transformers outperformed other models, including Mamba, due to the specialized attention heads.
\citet{phonebook} also discover that Transformers are superior to Mamba on tasks that require copying from the input context.
\citet{mamba_icl} point out that Mamba struggles to retrieve vectors from the context of multi-query associative recall (MQAR) \citep{MQAR}, while Transformers can easily handle it well.
Furthermore, \citet{Mamba_large_ex} conduct experiments on larger models (up to 8B parameters) with a broader range of tasks, discovering that when it comes to in-context learning and recalling information from text, although Mambas can contain the same knowledge as Transformers, it will be more difficult for them to directly copy useful information from history.

Although there has been some empirical exploration, the theoretical investigation concerning the above "free lunch" question still remains open to explore.
In this paper, inspired by the comparison between Mamba and linear attention mechanism, we first focus on Mamba's ability to perform our defined COPY operation, which is closely related to the ability to retrieve information from context.
Our theoretical results suggest that constant-sized Mamba may struggle with the COPY operation due to its fixed inference cost that does not scale with sequence length, whereas Transformers handle it more easily. However, if Mamba's model size scales linearly with the sequence length, it becomes capable of performing the COPY operation.
Further, following the framework established by \citet{hedi_CoT, hedi_linear_attention}, we explore Mamba’s capability to reason via Chain-of-Thought (CoT), which can be formulated as dynamic programming (DP) problems.
We find that to solve arbitrary DP problems, Mamba and Transformers seem to be on equal footing in terms of inference cost; however, Mamba may offer savings when dealing with $m$-locality DP problems like efficient Transformers \citep{hedi_linear_attention}.
Our results can be concluded as follows:
\begin{itemize}
	\item  Inspired by the connection between linear attention and the SSM module, we investigate Mamba’s ability to perform the COPY operation, showing that constant-sized SSM modules are less effective than attention in this task, unless the model size scales linearly with the sequence length (in Section \ref{sec:copy});
	\item When equipped with CoT, the total cost required by Mamba to solve arbitrary DP problems is comparable to that of standard and efficient Transformers. However, when the DP problems have locality properties, Mamba can bring savings in overhead compared to standard Transformers (in Section \ref{sec:cot});
	\item We conduct experiments on both the copy and CoT tasks, demonstrating Mamba’s limitations compared to Transformers in learning these tasks (in Section \ref{sec:ex}).
\end{itemize}

\begin{figure*}[t]
	\centering
	\includegraphics[scale=0.30]{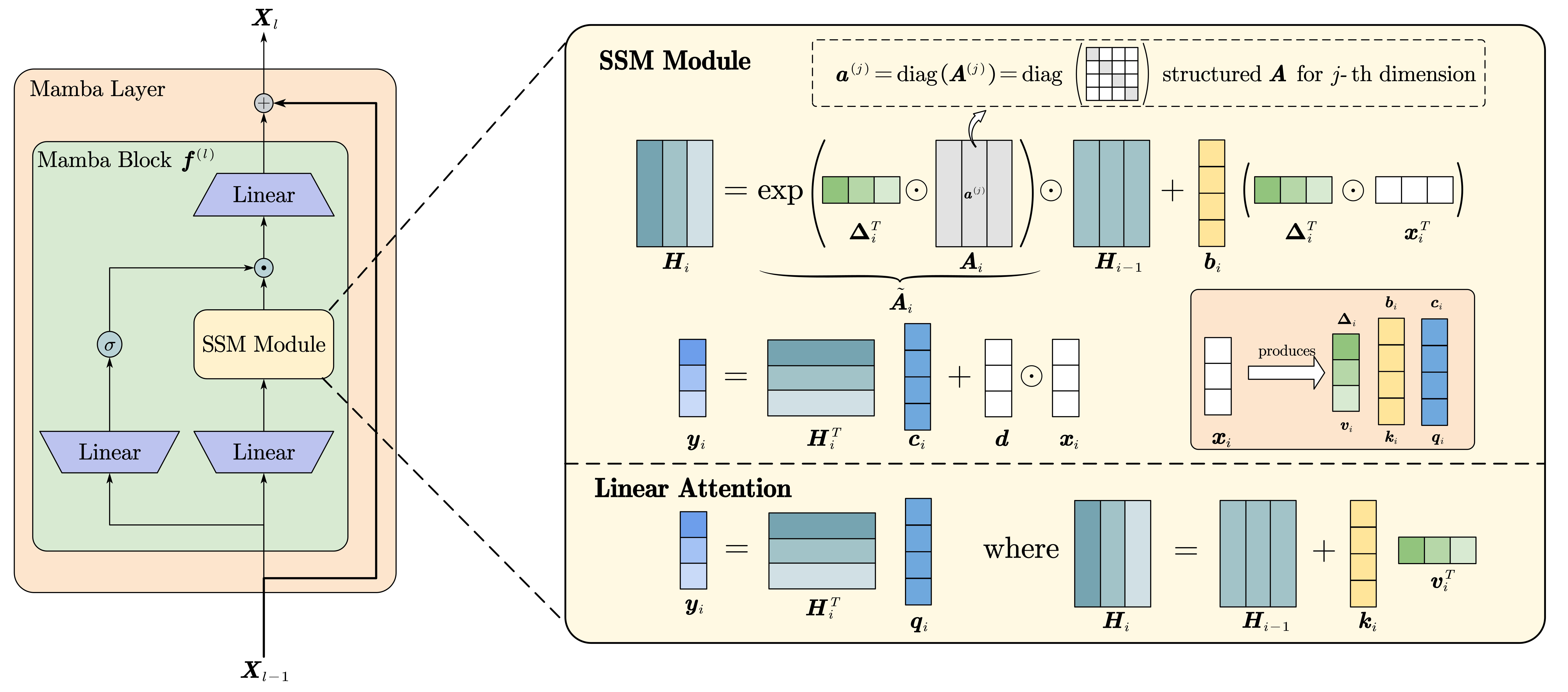}
	\caption{The illustration of the simplified Mamba layer we focus on. \textbf{Left Part}: A Mamba layer can be composed of a Mamba block with the residual connection; The Mamba block uses a gated MLP to control the output of the SSM module, where we call the branch with the SSM module as ``the SSM branch" while the other as ``the gated branch"; \textbf{Right Part:} The SSM module used in Mamba can be rewritten in a form similar to linear attention, where  $\vDelta_i$, $\vb_i$, and $\vc_i$ in SSM are all derived from the current $\vx_i$, similar to $\vv_i$, $\vk_i$, and $\vq_i$ in linear attention respectively.}
	\label{fig:mamba}
	\vspace*{-0.3cm}
\end{figure*}

\section{Preliminaries}\label{sec:mamba}
In this section, we introduce the Mamba structure that we focus on and its reformulated form firstly introduced by \citet{Mamba_thu_linear_attention}, which facilitates a better understanding of the connection between Mamba and linear attention.

\textbf{State Space Model:} The state space model (SSM) is inspired by the continuous system that maps a scalar input $x(t) \in \sR$ to its output $y(t) \in \sR$ through a high-dimensional hidden state $\vh\in \sR^{d_h}$ \citep{Mamba, Mamba_thu_linear_attention, Mamba_vision}. 
Specifically, this system can be written as:
\begin{align*}
	\vh'(t) &= \mA \vh(t) + \vb x(t), ~~y(t) = \vc^T\vh(t) + d  x(t),
\end{align*}
where $\mA \in \sR^{d_h \times d_h}$ denotes the evolution parameters, $\vb, \vc \in \sR^{d_h}$ are projection parameters and $d$ is a scalar parameter.
This continuous system can be discretized using zero-order hold (ZOH), resulting in a discrete version that can be used for neural networks.
In this process, $\mA, \vb$ will be transformed as $\overline{\mA}$, $\overline{\vb}$.
The discrete version can be written as:
\begin{align*}
	\vh_{i} &= \bA \vh_{i-1} + \bb x_{i}, ~~~~y_{i} = \vc^T\vh_{i} + d x_{i},
\end{align*}
where $\bA = \mathrm{exp}(\Delta \mA)$, $\bb = (\Delta \mA)^{-1}(\mathrm{exp}(\Delta \mA)-\mI)\cdot \Delta \vb \approx \Delta \vb$ and $\Delta \in \sR$ is a timescale parameter.
The matrix $\mA$ is often assumed to be structured, e.g., diagonal, resulting in structured SSMs \citep{S4, S4_2}. 

\textbf{Selective State Space Module:} To enhance the SSM, Mamba makes $\vb_i, \vc_i, \Delta_i$ dependent on different inputs $x_i$.
Specifically, $\mA$ is set to be diagonal resulting in that $\bA_{i} \vh_{i-1} = \pa_{i} \odot \vh_{i-1}$ where $\pa_i = \mathrm{exp}(\Delta_i \va)$, $\va = \mathrm{diag}(\mA)$ and $\odot$ denotes the element-wise product.
In addition, $\bb_i x_i = \Delta_i \vb_i x_i = \vb_i (\Delta_i \odot x_i)$.
Thus, this transformation ultimately results in:
\begin{align*}
	\vh_{i} &= \pa_i \odot \vh_{i-1} + \vb_i (\Delta_i \odot x_i), \\
	y_{i} &= \vc_i^T\vh_{i} + d \odot x_{i}.
\end{align*}
Furthermore, to extend the case of processing scalar inputs $x_i$ to vectors $\vx_i \in \sR^{d}$, Mamba performs the above operations on each dimension independently, which can be formalized as:
\begin{equation}\label{yssm}
	\begin{aligned}
		\mH_{i} &= \pA_i \odot \mH_{i-1} + \vb_i (\vDelta_i \odot \vx_i)^T, \\
		\vy_{i} &= \mH_{i}^T\vc_i + \vd \odot \vx_{i},
	\end{aligned}
\end{equation}
where we have $\pA_i = [\pa_{i}^{(j)}]_{j=1}^d \in \sR^{d_h\times d}$ , $\vb_i = \mW_\vb \vx_{i} \in \sR^{d_h}$, $\vc_i = \mW_\vc \vx_i \in \sR^{d_h}$ and $\vDelta_i = \mathrm{Softplus}(\mW_{\vDelta}^2\mW_{\vDelta}^1\vx_i) \in \sR^{d}$.
Thus, given the input $\mX = [\vx_i]_{i=1}^{N} \in \sR^{d \times N}$, we denote the output of the SSM module in Mamba as $\mY = \mathrm{SSM}(\mX)$
where $\mY = [\vy_i]_{i=1}^N$ and $\vy_i$ follows Eq~(\ref{yssm}).
This formalization was introduced by \citet{Mamba_thu_linear_attention} to build a bridge between Mamba and linear attention and here we follow this form.

\noindent \textbf{Mamba Layer:}
Given an input sequence $\mX = [\vx_{i}]_{i=1}^{N} \in \sR^{d \times N}$,  it will be processed by stacked Mamba layers, each comprising a residual connection and a Mamba block $\vf^{(l)}:\sR^{d} \rightarrow \sR^{d}$. 
The output of the $l$-th layer can be formulated as
%
\begin{align}
	\mX^{l+1} &= \mX^{l} + \vf^{(l)}(\mX^{l}),   \label{mlayer} \\
	\vf^{(l)}(\mX^{l}) &= \mW_3^{l} \cdot \mathrm{SSM}(\mZ_1^l) \odot \sigma(\mZ_2^l), \label{mblock}
\end{align}
where $\mZ_1^l = \mW_1^l \mX^l + \vb_1^l$, $\mZ_2^l = \mW_2^l \mX^l + \vb_2^l$ and $\sigma(\cdot)$ denotes $\silu$ activation function.
A Mamba block combines the output of the SSM module\footnote{To avoid confusion, we clarify that the SSM module here is specifically the Selective State Space Module used in Mamba, which is followed throughout the rest of the paper.}
Here we call the branch with the SSM module as "the SSM branch" and the other as "the gated branch". 
To facilitate a clearer analysis, we restrict our attention to the core components of Mamba, namely the SSM module and the gated MLP, and leave out other architectural details, as illustrated in Figure~\ref{fig:mamba}.

\section{Can Mamba always Perform COPY Perfectly?}\label{sec:copy}
In this section, we begin by interpreting the reformulated SSM module as in Section~\ref{sec:mamba} as a special case of linear attention, and then explore Mamba's ability to perform COPY operations during inference, which is crucial to retrieving contextual information during model's reasoning.

\subsection{Viewing Mamba as linear attention}\label{sec:attention}
The attention mechanism is the key to the success of Transformers. 
Recent works has explored the relationship between Mamba and attention mechanisms particularly the linear attention \cite{Mamba_thu_linear_attention,Mamba2,attention_ssm_rnn}.
The linear causal attention can be formalized as:
\begin{equation}\label{eq:linear_atten}
	\vy_{i} = \sum_{j=1}^{i} \vv_j \vk_j^T \vq_i = \sum_{j=1}^{i} (\vq_i^T \vk_j) \vv_{j} = \sum_{j=1}^{i} a_{ij} \vv_j,
\end{equation}
where $\vq_i$, $\vk_i$, $\vv_i$ are usually interpreted as query, key, value respectively and $a_{ij}$ denotes the attention scores of the $i$-th token to the $j$-th one.
In attention mechanisms in Transformers, there exists $a_{ij}>0$ for all $j \le i$ and $\sum_{i=1}^{j} a_{ij} = 1$, which can be implemented by Softmax function.

On the other hand, given the input sequence $[\vx_i]_{i=1}^N$, the output of the SSM module formulated as Eq~(\ref{yssm}) will have the following form when we set $\mH_0$ and $\vd$ to be zeros:
\begin{equation}\label{eq:yssm}
	\vy_i =  (\vDelta_i \odot \vx_i) \vb_i^T \vc_i + \sum_{j=1}^{i-1} \left[ \mathbf{\Pi}_j \odot (\vDelta_j \odot \vx_j) \vb_j^T \right] \vc_i,
\end{equation}
where $\mathbf{\Pi}_j = \pA_{i} \odot \pA_{i-1} \odot \dots \odot \pA_{j+1}$.
We notice that since in practice all elements of $\vDelta$ are positive and $\mA$ is set to be negative \cite{Mamba,Mamba2,Mamba_thu_linear_attention}, the elements of $\pA_i$ in Eq~(\ref{yssm}) will belong to the interval $[0,1]$ as $\pa = {\rm diag}({\rm \exp}(\Delta\mA))$.
To simplify our analysis, we replace the matrix $\pA_i$ with a forgetting coefficient $a_i\in [0,1]$ (i.e., considering the case where all elements of $\pA_i$ are the same as $a_i$ \cite{Mamba2}).
In fact, the subsequent analysis can be easily extended to the non-simplified case.
Then, Eq~(\ref{eq:yssm}) can be rewritten as  
\begin{equation}\label{eq:atten}
	\vy_i = \sum_{j=1}^{i}\alpha_j (\vc_i^T\vb_j) (\vDelta_j \odot \vx_j) ,
\end{equation}
where $\alpha_j = \Pi_{k=j+1}^{i}a_{k}$ for $j \le i-1$ and $\alpha_i = 1$.
We call $\alpha_j$ as the cumulative forgetting coefficient.  
In this form, we can observe that it bears similarities to linear attention without normalization in Eq~({\ref{eq:linear_atten}), where $(\vDelta_j \odot \vx_j)$, $\vb_j$, $\vc_i$ corresponds to $\vv_j$, $\vk_j$ and $\vq_i$ respectively and $\vc_i^T\vb_j$ acts like attention scores $a_{ij}$.
Considering $\alpha_{j-1} \le \alpha_j$ and $\alpha_j \in [0,1]$ for all $j \le i$, the main difference is that each term in Eq~(\ref{eq:atten}) is weighted by a coefficient $\alpha_j$ to achieve the forgetting of inputs at longer distances while the attention mechanism in Transformers uses the constraints for attention scores imposed by Softmax to make sure the scaling of outputs.

\begin{figure}[t]
	\centering
	\includegraphics[scale=0.43]{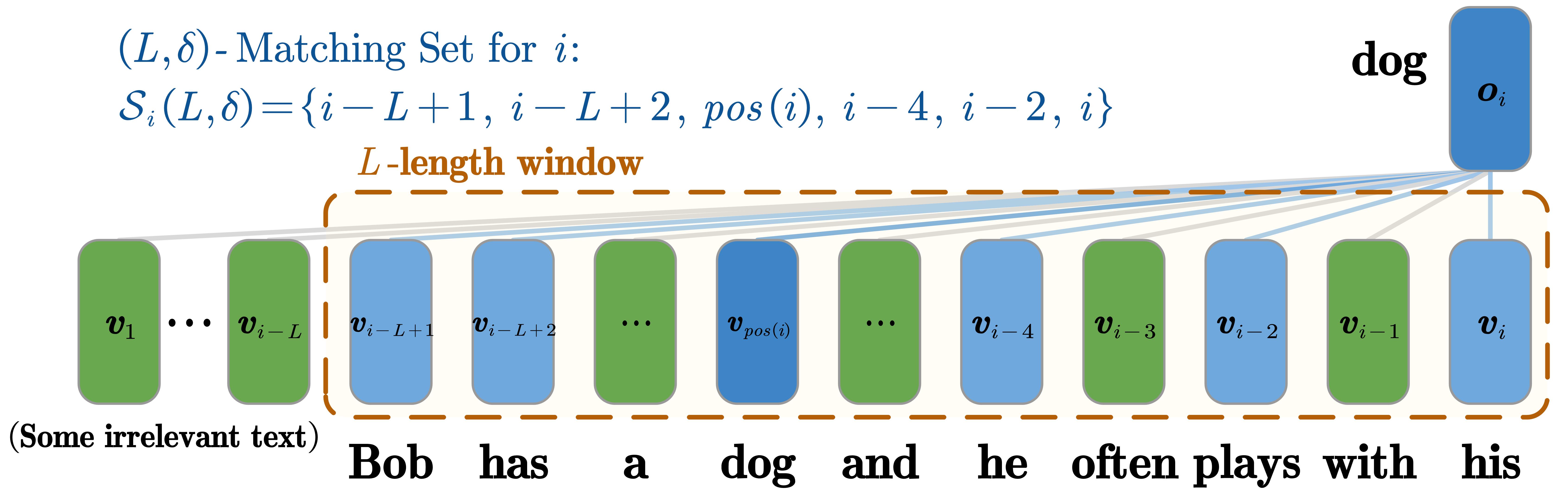}
	\caption{An example for COPY operation and  ($L,\delta$)-matching set. We expect the output at position $i$ to be the historical record (value) corresponding to ``dog".
		The historical records belonging to the ($L,\delta$)-matching set are labeled in blue, which are more relevant to the output $\vo_i$ based on the attention scores $|\vc_i^T\vb_j| \ge \delta$.}
	\label{fig:copy}
	\vspace*{-0.3cm}
\end{figure}

\subsection{Limitation of Mamba to Perform COPY}\label{sec:trade-off}
Based on the observation of the connection between the SSM module and attention mechanism, we investigate the capability of Mamba to recover historical inputs, which is foundational for the model to process information based on context. We define the COPY operation as follows:
\begin{definition}[COPY Operation]
	For a given SSM module and input sequence $\vx_1, \vx_2, \dots, \vx_{N}$, we denote $\vv_i = \vDelta_i \odot \vx_i$ as historical records.
	Then the output of COPY operation is a sequence of vectors $\vo_1, \vo_2, ..., \vo_N$ with $\vo_i = \vv_{pos(i)}$ where $pos(i)$ is the position we want to copy.
\end{definition}
As stated earlier, for a given SSM module, $(\vDelta_j \odot \vx_j)$, $\vb_j$, $\vc_i$ corresponds to value, key and query respectively. 
Thus from the perspective of linear attention, the COPY operation for position $i$ aims to retrieve the interested value located at $pos(i)$.
Intuitively, there exists some historical records $\vv_j$ that are more relevant to the current query $\vc_i$ than others, which can be described as:
\begin{definition}[($L,\delta$)-Matching Set] 
	For a given SSM module and input sequence $\vx_1, \vx_2, \dots, \vx_{N}$, the ($L,\delta$)-matching set for $\vx_i$ is defined as $\gS_i(L,\delta) =  \{ j \mid |\vc_i^T \vb_j| \ge \delta,  i-L < j\le i \}$.
\end{definition}
The ($L,\delta$)-local matching set describes the positions of historical keys $\vb_j$ that is highly relevant to the current query $\vc_i$ within a local window of length $L$, that is, the "attention scores" $|\vc_i^T\vb_j|$ is lower-bounded by $\delta$.
An illustration for COPY operation and  ($L,\delta$)-matching set is shown in Figure~\ref{fig:copy}.
We then make the following assumption:
\begin{assumption}\label{assum:copy}
	For a given SSM module and some input sequence $\vx_1, \vx_2, \dots, \vx_{N}$, the following conditions holds: 
	\begin{itemize}
		\item  There is some $L\in \sN$ and $\delta \in \sR$ such that for any $i \in [N]$, $S_i(L,\delta)$ exists and $pos(i) \in \gS_i(L,\delta)$. In addition, for any $j \notin \gS_i(L,\delta)$, $|\vc_i^T \vb_j| < \delta$.	
		\item For any $i \in [N]$,  $\sum_{j \in \gS_i(L,\delta)}\alpha_j|\vc_i^T \vb_j| < 1$.
	\end{itemize}
\end{assumption}
The first condition requires that the positions to be copied lies in $\gS_i(L,\delta)$, which is intuitive since $\gS_i(L,\delta)$ captures past positions most relevant to the output $\vo_i$, and $pos(i)$ should be included in this set as its highly relevance.
The second condition in Assumption~\ref{assum:copy} imposes a constraint on the scaling of the output caused by the records in $\gS_i(L,\delta)$, that is, $\|  \sum_{j \in \gS_i(L,\delta)}\alpha_j\vc_i^T \vb_j \vv_j \| < 1$ when $\|\vv_j \| \le 1$.
Next, we explore the condition that can enable a given SSM module to approximate the COPY operation.
Below, we provide our result:
\begin{theorem}[Approximate COPY operation with constant-size SSM module]\label{constant_copy}
	Given a SSM module with constant size and the input sequence $\vx_1, \vx_2, \dots, \vx_{N} \in [-M, M]^{d}$ such that Assumption \ref{assum:copy} holds, then for any $\epsilon > 0$, the SSM module can approximate COPY operation at some position $i$, that is, $\| \vy_i - \vo_i \|_{\infty} \le \epsilon$  if there is $\vc_{i}^T\vb_{pos(i)} \ge c ( \frac{1}{a_{\min}})^{L-1} + d $ where $a_{\min} = \min_{pos(i) < j \le i} a_j$ and $c, d$ are constants related to $\epsilon$, $\delta$.
\end{theorem}
The proof can be seen in Appendix~{\ref{app:constant_copy}}. 
Theorem \ref{constant_copy} shows that, relative to the sequence length $N$, a constant-size SSM module can approximate COPY within error $\epsilon$ if the attention score $\vc_{i}^T\vb_{pos(i)}$ is lower bounded by $c ( \frac{1}{a_{\min}})^{L-1} + d$.
This means that larger forgetting coefficients after $pos(i)$ help retain $\vv_{pos(i)}$ to some extent, making it easier for the attention score to meet the lower bound and thereby enabling the COPY operation.
Nevertheless, we note that since $a_{\min} < 1$,  achieving the COPY operation requires the attention score to grow exponentially with $L$, where $L$ is the distance between the current position $i$ and the farthest highly relevant historical record.
This renders the condition difficult to satisfy.
However, we will show that under similar assumptions, a constant-size attention module in Transformers can perform COPY under less restrictive conditions, which can be described as follows:
\begin{theorem}[Approximate COPY operation with constant-size attention module]\label{constant_copy_attention}
	Given a attention module with constant size and input sequence $\vx_1, \vx_2, \dots, \vx_{N} \in [-M, M]^{d}$ such that Assumption \ref{assum:attention copy} holds, then for any $\epsilon > 0$, the attention module can approximate the COPY operation at some position $i$, that is, $\| \vy_i - \vo_i \|_{\infty} \le \epsilon$  if there is $\vq_i^T \vk_j \ge \log\tilde{L} + c$ where $\tilde{L} = \max\{L, i - |\gS_{i}(L,\delta)|\}$ and $c$ is a constant related to $\epsilon$, $\delta$.
\end{theorem}
More details can be seen in Appendix~\ref{app:attention_copy}.
Theorem \ref{constant_copy_attention} shows that a constant-size attention module can achieve COPY if $\vq_i^T \vk_{pos(i)}$ (corresponding to $\vc_i^T \vb_{pos(i)}$ above) is lower bounded by $\log L$, which is much easier to satisfy than the condition for the SSM module.
An intuitive explanation is that despite having a constant parameter size, the attention module still maintains $O(N)$ cost when inference, allowing it sufficient capacity to store historical records and retrieve the desired one.
In contrast, the SSM module with a constant size (thus $O(1)$ inference cost) can easily have the interested records overwhelmed by related but irrelevant information, making COPY more difficult to achieve.
This naturally leads to the question: can we make COPY easier by increasing the size of the SSM module? In the following, we show that when its size scales linearly with the sequence length $N$, it is always possible for the SSM module to perform COPY.
\begin{theorem}[Perform COPY operation with linear-scaling size]\label{linear_copy}
	Given the input sequence $\vx_1, \vx_2, \dots, \vx_{N} \in [-M, M]^{d}$, there exists a SSM module with size $O(N)$ that can perform the COPY operation, that is, $\vy_i = \vo_i$ for any $i \in [N]$.
\end{theorem}
The proof can be found in Appendix~\ref{app:linear_copy}.
Theorem~\ref{linear_copy} is based on a simple intuition: when the model size grows linearly with the length of input sequence, the model will have enough space to store these historical records and therefore can retrieve them.
However, it should be noted that such a linear-size SSM module will have the same cost as attention module in Transformers when inference, that is, $O(N)$ at each step. 
Thus from this perspective, Mamba does not offer savings when facing the COPY operation.
Based on this observation, we will elaborate in Section~\ref{sec:cot} that when faced with CoT reasoning tasks modeled by dynamic programming problems, Mamba incurs the same order of overhead as Transformers.


\section{Mamba equipped with CoT to Solve DP}\label{sec:cot}

Chain-of-Thought (CoT) is regarded as a powerful approach to enhancing a model's reasoning ability \citep{CoT}.
It allows the model to solve complex reasoning problems by decomposing them into a sequence of simpler subproblems. 
During inference, the model needs to retrieve useful contexts from the reasoning chain, which is closely related to its ability to perform the COPY operation, and incrementally use these contexts to produce the final output. 
Such a reasoning process can be modeled as a dynamic programming (DP) problem \citep{hedi_CoT, hedi_linear_attention}, characterized by an input sequence, a state space, a transition function, and an aggregation function, each of which is described below.

\noindent \textbf{$\bullet$~Input sequences:} We use $\{s^{(1)}, s^{(2)}, \dots, s^{(N)}\}$ to denote the input sequences and the vector $\vn = \left[|s^{(1)}|, |s^{(2)}|, \dots, |s^{(N)}|\right]^T$ to describe the size of the problem, where $|s^{(i)}|$ denotes the length of the $i$-th sequence.

\noindent \textbf{$\bullet$~State Space:} A DP problem can be decomposed into a series of sub-states to solve, forming a state space $\gI_{\vn}$ whose size depends on the problem size $\vn$.
Each state \( i \in \gI_{\vn}\) represents an intermediate value ${\rm dp}(i)$ to compute, and \( i \prec j \) means that state \( i \) needs to be solved before state \( j \).
The function $f_{\gI}: \gI_{\vn} \rightarrow \gI_{\vn}$ defines the next state if $j$ is the next state to solve after state $i$.

\noindent \textbf{$\bullet$~Transition function:} The intermediate DP values can be calculated by a transition function $f_\gT$ as $\mathrm{dp}(i) = f_{\gT}(\vn, \vs, \{ (j, \mathrm{dp}(j)) : j \prec i  \})$ where $\vs$ denotes the concatenation of all tokens from the input sequences\footnote{We use $s^{(i)}$ to denote the $i$-th input sequence while $\vs_i$ to denote the $i$-th input token, where $\vs$ is all input tokens from the concatenated input sequences.}.
This can be rewritten as
$\mathrm{dp}(i) = f_{\gT}\left( \vn, ~\{ \vs_j: j \in \gI_{i} \},~ \{ \mathrm{dp}(k) : k \in \gV_{{\rm dp}(i)}  \}\right)$ where $\gI_i $ and $\gV_{{\rm dp}(i)}$ are the sets of input tokens indices and DP values to solve state $i$ respectively.

\noindent \textbf{$\bullet$~Aggregation function:} To produce the final answer, the aggregation function collects the required intermediate DP values and calculate the final result as $A = f_{\gA}(\{{\rm dp}(i): i \in \gA_\vn\})$ where $A$ denotes answer and $\gA_\vn$ is the set of DP values needed in the aggregation according to the problem size $\vn$.

We consider how Mamba layers, as defined in Eq.~(\ref{mlayer}), incrementally generate solutions to DP problems using CoT. The generated sequence follows the format:
\begin{equation*}
	\begin{aligned}
		&s^{(1)} ~|~ s^{(2)} ~|~ \dots ~|~ s^{(N)} ~|~ \left(i_1, \mathrm{dp}(i_1) \right)~ \left(i_2, \mathrm{dp}(i_2) \right) \\ 
		&\left(i_3, \mathrm{dp}(i_3)\right) \dots	 ~\left(i_{|\gI_{\vn}|}, \mathrm{dp}(i_{|\gI_{\vn}|})\right) |~A
	\end{aligned}
\end{equation*}
where the input sequence is separated using the symbol $|$ as a delimiter.
An classic example of DP problems is the Longest Increasing Subsequence (LIS) problem, whose goal is to find the length of the longest increasing subsequence of a given integer sequence.
Following the above form, an example of the CoT output sequence can be 
\[
\underbrace{13\ 32\ 12\ 39\ 84}_{\text{input sequence } s^{(1)}}
\quad | \quad
\underbrace{1\ 2\ 1\ 3\ 4}_{\text{DP values}}
\quad | \quad
\underbrace{4}_{\text{final answer $A$}},
\]
where there is only one input sequence $s^{(1)}$ for this problem.
More examples for DP problems can be seen in Appendix~\ref{app:dpexample}.
The reasoning process of LLMs in real scenarios can generally be modeled in the form described above.
Based on this formulation, we present the following result showcasing Mamba’s ability to solve DP problems when equipped with CoT:
\begin{theorem}[Solve DP problems with CoT]\label{DP_Mamba}
	Considering any DP problem and given input sequences that satisfies Assumption~\ref{assum:dp}, for any integer $T \in \sN$, there exists several Mamba layers with size $O(T)$, such that the answer generated by the Mamba layers will be correct when the length of the answer is no more than $T$.
\end{theorem}
More details can be seen in Appendix~\ref{app:cot}. 
The intuition behind Theorem \ref{DP_Mamba} is that when the size of the Mamba layers scales linearly with $T$, the model gains sufficient capacity to retrieve useful intermediate states from the reasoning chain for the next inference step, similar to the behavior described in Theorem~\ref{linear_copy}.
In this case, each CoT step in Mamba incurs an $O(T)$ cost, resulting in a total inference cost of $O(T^2)$.
Similarly, as shown by \citet{hedi_CoT}, a constant-sized Transformer can also solve any DP problem with CoT, but due to the per-step attention cost scaling linearly with $T$, the total inference cost remains $O(T^2)$.
While for efficient Transformers, \citet{hedi_linear_attention} reached similar conclusions.
Thus, from this prespective, Mamba also does not offer additional savings:
Due to the limitation of constant inference capacity, Mamba may need to increase its model size to achieve performance comparable to that of a constant-sized Transformer, which in turn leads to higher inference cost.

It seems disappointing that, like efficient Transformers, Mamba does not reduce overhead for general DP problems.
However, we argue that when DP problems exhibit favorable local properties\citep{hedi_linear_attention}, Mamba has the potential to achieve efficiency.
Assuming that the output of CoT can be written as $\vo_1, \vo_2, \dots, \vo_T$, if $\vo_i = f(\{\vo_j:  i - m \le j < i \})$ for any $i \in [T]$, that is, $\vo_i$ only depends on at most $m$ preceding intermediate results, then we call the DP problem is $m$-locality DP problem.
As illustrated in Appendix~\ref{app:m-locality}, the common arithmetic task can be viewed as such $m$-locality DP problem.
Then we present the following result:
\begin{theorem}[Solve $m$-locality DP problems with CoT]\label{localDP_Mamba}
	Consider any $m$-locality DP problem and given input sequences that satisfies Assumption~\ref{assum:dp}, for any integer $T \in \sN$, there exists several Mamba layers with size $O(m)$, such that the answer generated by the Mamba layers will be correct when the length of the answer is no more than $T$.
\end{theorem}
The proof can be seen in Appendix~\ref{app:cot_local}. 
Theorem~\ref{localDP_Mamba} shows that when handling $m$-locality DP problem with CoT, the needed size of Mamba depends on the problem's locality. The cost for each step becomes a constant $O(m)$ and the total cost becomes $O(mT)$ rather than $O(T^2)$ leading to savings in cost when $T$ is much larger than $m$.

\begin{figure}[t]
	\begin{subfigure}[t]{0.19\linewidth}
		\centering
		\includegraphics[scale=.19]{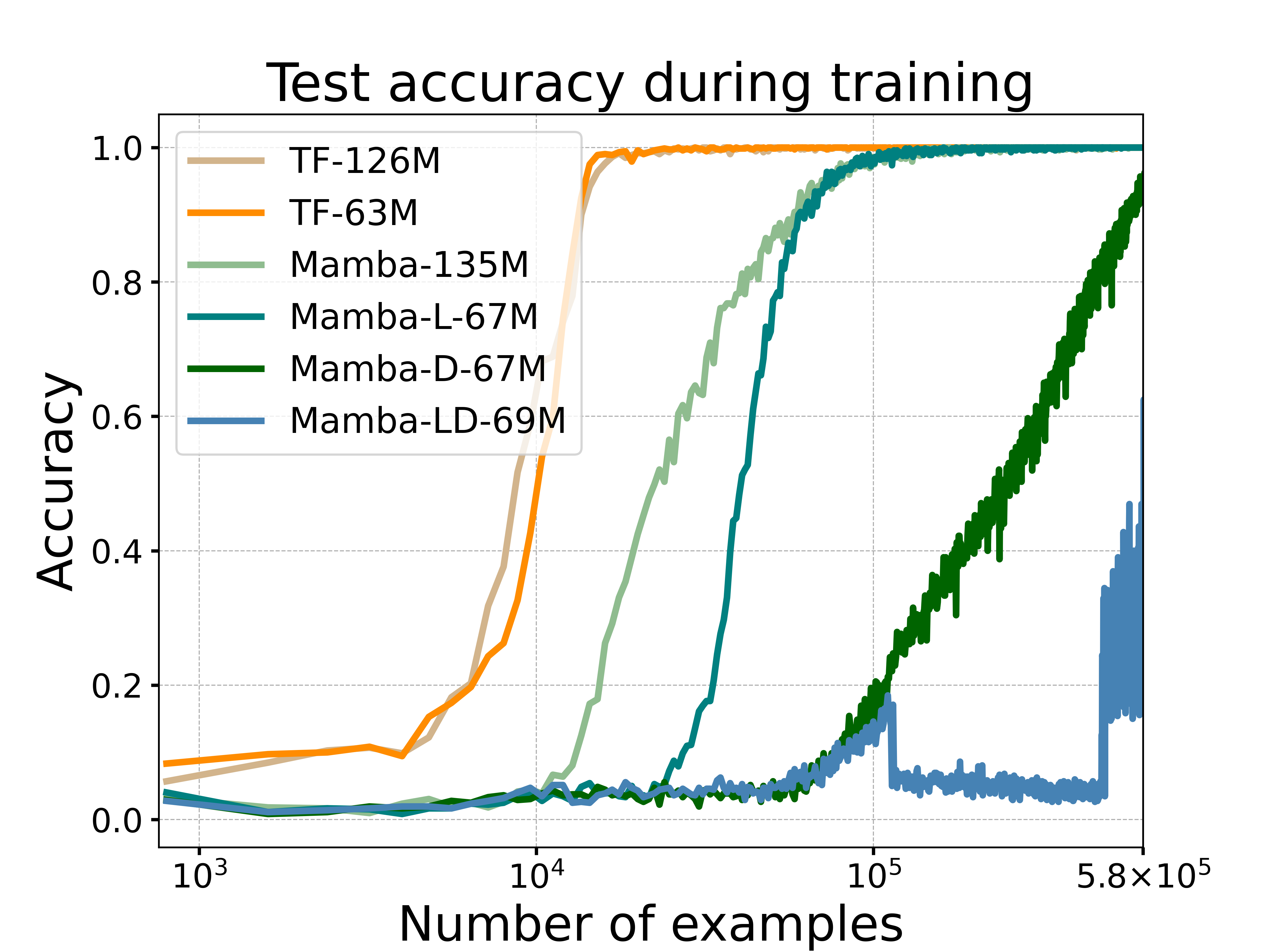}
	\end{subfigure}
	\hspace{2.3cm}
	\begin{subfigure}[t]{0.19\linewidth}
		\centering
		\includegraphics[scale=.19]{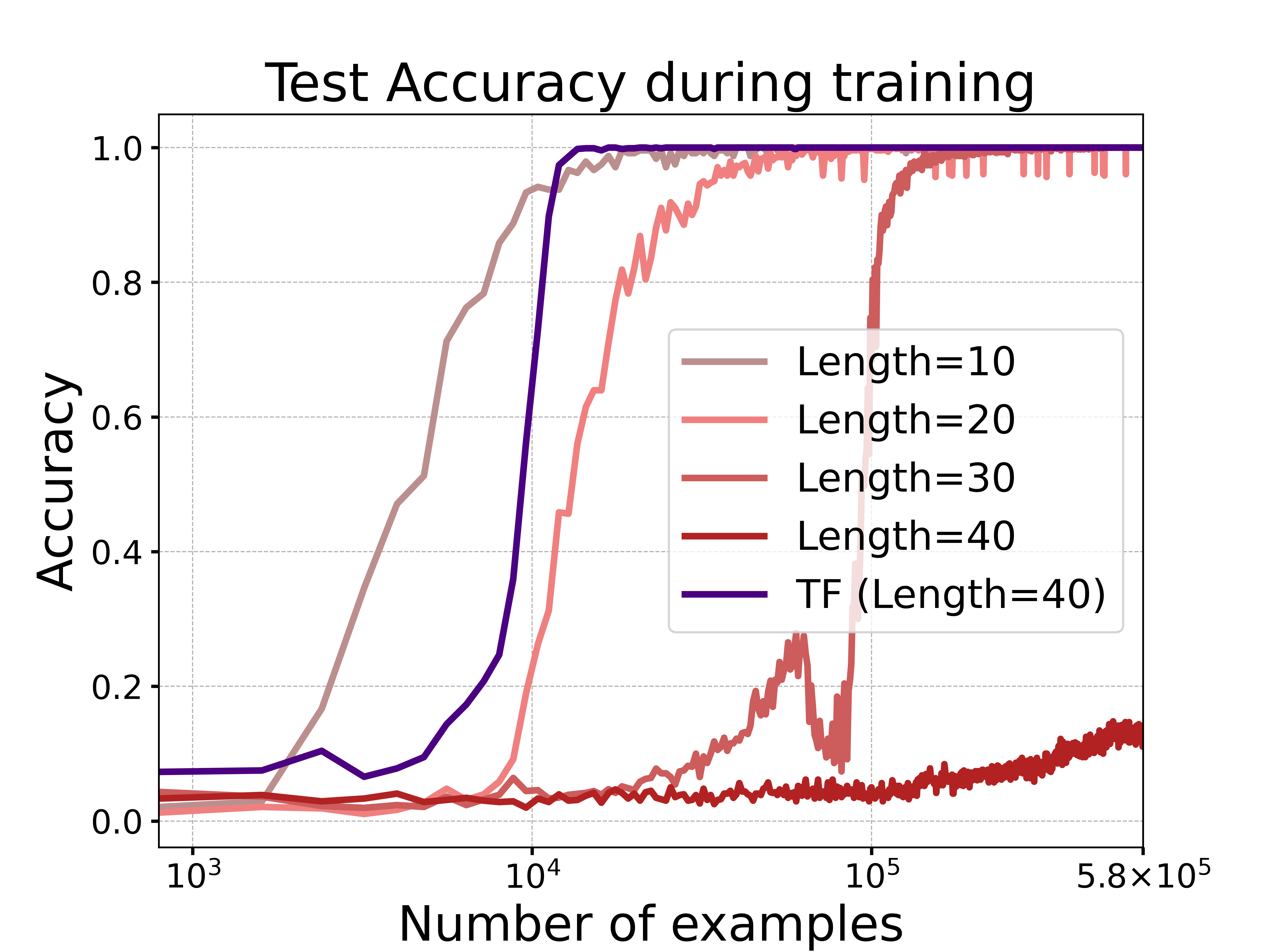}
	\end{subfigure}
	\vspace{-0.cm}
	\caption{\textbf{Left:} Accuracy during training of models with different sizes on the copy task. 
	\textbf{Right:} The performance of Mamba when the length of the input sequence to be copied is changed.}
	\vspace*{-0.3cm}
	\label{fig:ex}
\end{figure}

\section{Experimental Results}\label{sec:ex}

In this section, we conduct experiments to further illustrate our findings.

\textbf{Experiments on copy tasks:}
We evaluate models on the copy task introduced by \citet{phonebook}, where the goal is to repeat an input string exactly. During training, input lengths are uniformly sampled from $[N_{\min}, N_{\max}]$, with characters drawn randomly from the alphabet. At test time, models copy strings of fixed length $N_{\max}$, and accuracy is measured by the proportion of correctly copied characters. More details are in Appendix~\ref{app:ex}, and results are shown in Figure~\ref{fig:ex}.

We first compare models of different sizes on fixed-length strings (left of Figure~\ref{fig:ex}, where Transformer is denoted as TF). 
We find that both TF-126M and Mamba-135M eventually learn the copy task, but TF converges much faster. 
For smaller models, TF-63M is hardly affected even with half the layers.  
In contrast, smaller Mamba variants often struggle: (i) \textbf{reducing layers} (Mamba-L-67M) slows learning; (ii)  \textbf{reducing the hidden size} (Mamba-D-67M) even slower; and (iii) \textbf{reducing the hidden size while increasing layers} (Mamba-LD-69M) finally fails to learn the task within finite examples and the training becomes unstable, which indicates that the hidden size has a greater impact on Mamba's performance.
This confirm our findings in Section \ref{sec:copy} that Mamba indeed finds it harder to learn the copy task compared to Transformers.

Furthermore, we change the maximum length $N_{\max}$ while maintaining the model size, as shown in the right of Figure \ref{fig:ex}. 
As the task becomes more challenging (increasing the length $N_{\max}$), Mamba requires more training examples to successfully learn the task and fails to learn within finite examples when $N_{\max} = 40$.
In contrast, Transformer can still learn quickly and maintain stability even at $N_{\rm max} = 40$, which again indicate that Transformer outperforms Mamba in executing copy operations as in Section \ref{sec:copy}.

\begin{figure}[t]
	\begin{subfigure}[t]{0.19\linewidth}
		\centering
		\includegraphics[scale=.19]{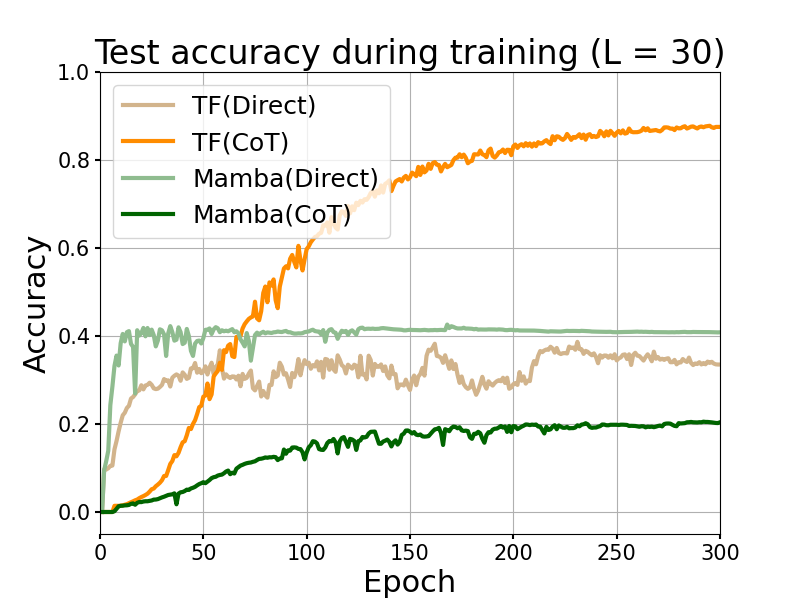}
	\end{subfigure}
	\hspace{2.3cm}
	\begin{subfigure}[t]{0.19\linewidth}
		\centering
		\includegraphics[scale=.19]{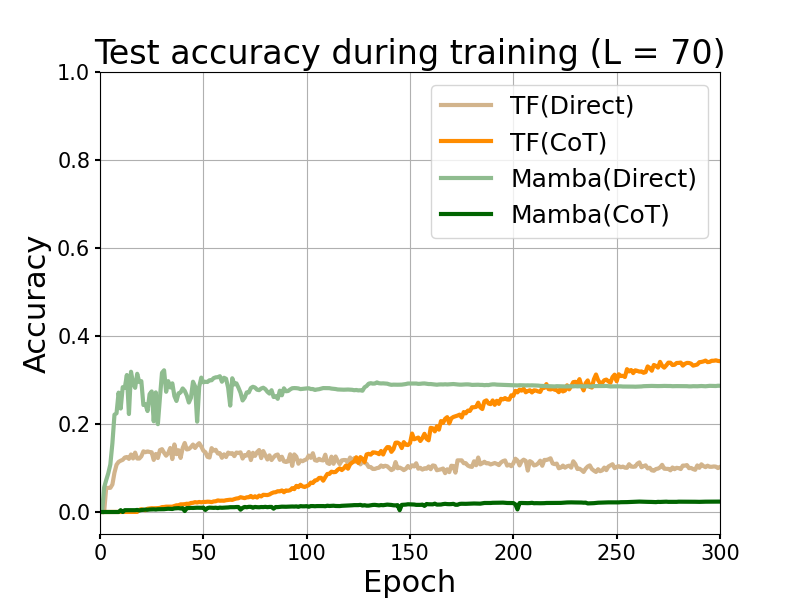}
	\end{subfigure}
	\vspace{-0.cm}
	\caption{Accuracy during training when the task length $L = 30/70$ and $d = 256$ (TF denotes Transformer).}
	\vspace*{-0.3cm}
	\label{fig:1}
\end{figure}

\begin{figure}[t]
	\begin{subfigure}[t]{0.19\linewidth}
		\centering
		\includegraphics[scale=.19]{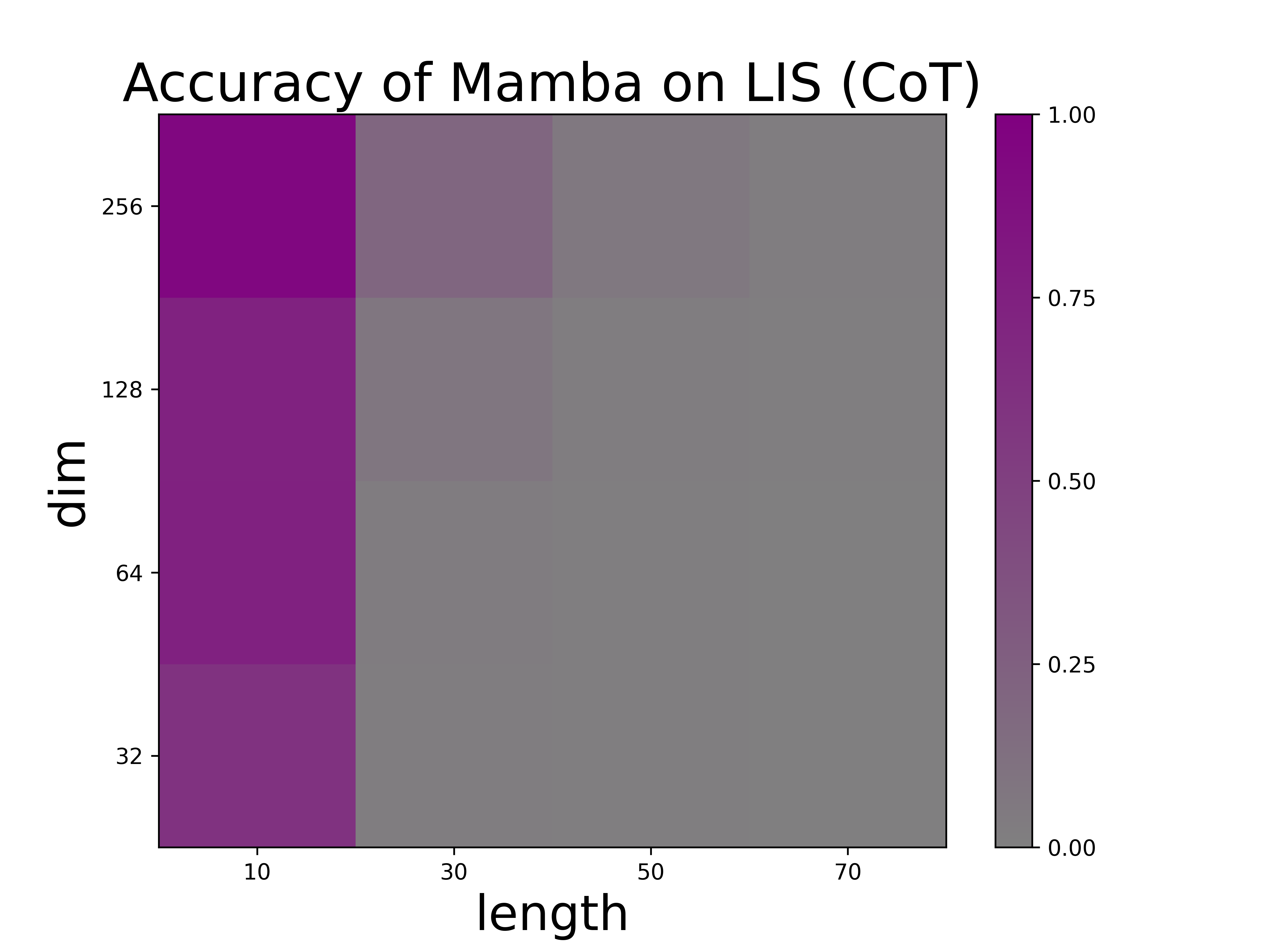}
	\end{subfigure}
	\hspace{2.3cm}
	\begin{subfigure}[t]{0.19\linewidth}
		\centering
		\includegraphics[scale=.19]{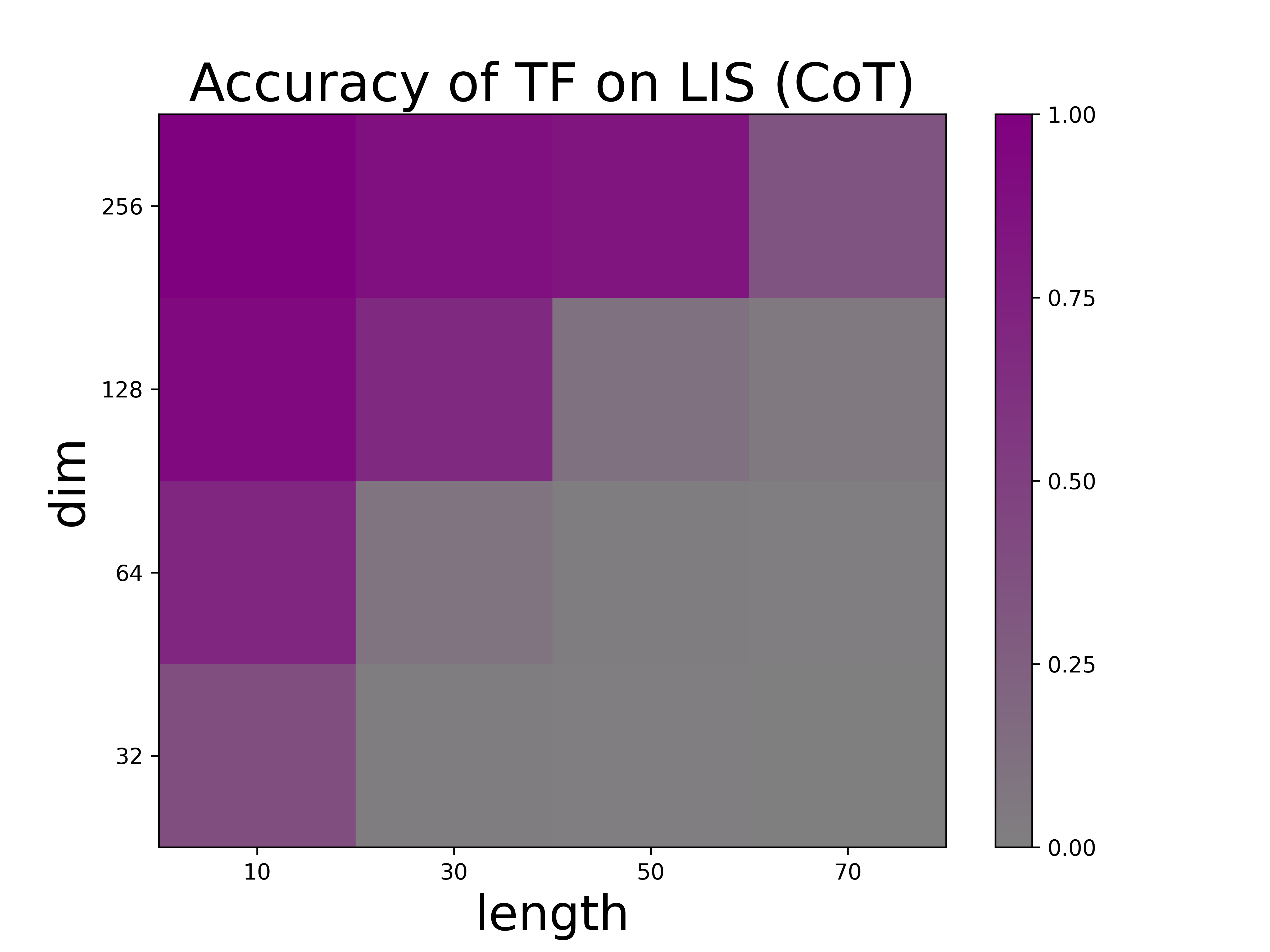}
	\end{subfigure}
	\vspace{-0cm}
	\caption{Accuracy of Mamba and Transformer under different task lengths and model sizes.}
	\vspace*{-0.3cm}
	\label{fig:2}
\end{figure}

\textbf{Experiments on CoT Task:}
In addition, we further evaluate model performance on solving DP problems, focusing on the  classic Longest Increasing Subsequence (LIS) problem following the setup of \citet{hedi_CoT}.
We examined two scenarios: (i)~\textbf{Direct:} the model is trained to directly output the final answer;~(ii)~\textbf{CoT:} the model is trained to output the answer through CoT reasoning, where it is required to generate both the correct reasoning steps and the final answer.
We investigate different task lengths $L$ (problem size) and model dimensions $d$. 
For Mamba, we adjust the number of layers to match or slightly exceed the size of the Transformer within the same $d$, with roughly two Mamba layers corresponding to one Transformer layer. More details can be seen in Appendix~\ref{app:ex}.

First, we present the accuracy curves during training in Figure \ref{fig:1}.
Under the Direct setting, Mamba outperforms Transformer due to the relatively short task length. 
However, under the CoT setting where significantly longer sequences are required, Transformer consistently performs better.
Notably, unlike Transformers, Mamba performs even worse under the CoT setting than in the direct setting. This is because Mamba’s constant-sized inference capacity limits its ability to handle long reasoning chains, whereas Transformers—with overhead that grows with sequence length—are better equipped to process such information.

Furthermore, in Figure \ref{fig:2}, we present results of Mamba and Transformer under different CoT settings. 
It can be seen that when the two models have comparable sizes, Mamba consistently underperforms Transformer, which supports our analysis in Section \ref{sec:cot}: while a constant-sized Transformer can effectively solve DP problems with CoT, Mamba with comparable size may struggle to do so unless given greater capacity.
These findings indicate that Mamba can not always get a free lunch: the inference cost that does not scale with sequence length may limit Mamba’s ability to solve certain tasks.

\section{Related work}
\textbf{SSMs and Attention Mechanism: }  The attention mechanism is a core component of LLMs \cite{transformer}. Drawing connections between SSMs and attention is a fascinating direction as it not only aids in our understanding of the Mamba structure but also facilitates the transfer of well-established acceleration techniques from attention mechanisms to Mamba \cite{yang2023gated,Mamba2,Mamba_thu_linear_attention,attention_ssm_rnn}. 
Based on observations of the similarities between them, \citet{Mamba2}~proposed the state space dual (SSD) layer based on SSMs to achieve significant improvements in training efficiency.
Particularly, \citet{Mamba_thu_linear_attention}~reformulate the structure of SSMs to establish links with linear attention, aiming to find the key factors behind success in vision tasks.
We follow this convenient reformulation to explore Mamba's ability to perform the COPY operation.

\textbf{Comparisons between Transformers and Mamba:} More recent works compare the performance of Mamba and Transformers across various tasks from experimental and theoretical perspectives.
\citet{phonebook}~find that Transformers significantly surpass SSMs when facing tasks related to copying and retrieving information from context.
In addition, \citet{mamba_icl}~investigate Mamba's capability for in-context learning and demonstrate that Mamba outperforms Transformers in sparse parity learning while it is weaker in tasks involving non-standard retrieval functionality.
Similarly, \citet{Mamba_large_ex} conduct experiments on a larger scale and find that Mamba lag behind Transformers in tasks that require strong copying and long-context reasoning.
Our experiments reference the setups of these works and conduct similar investigations.
From the theoretical perspective, \citet{merrill_mamba}~demonstrate that similar to Transformers, Mamba is also unable to solve state tracking problems such as permutation composition while we focus on the different task of performing COPY operation.
\citet{phonebook}~investigate the ability of generalized SSMs to replicate entire sequences that satisfy some distribution and providing a lower bound for their state space memory whereas our work focuses on analyzing the impact of the distance of the specified token to be copied on the output error from the numerical approximation perspective in Theorem~\ref{constant_copy}.
Additionally, \citet{MQAR}~use communication complexity to show that recurrent models require at least $\Omega(N)$ to solve Multi-Query Associative Recall (MQAR) tasks, which is a lower bound guarantee.
In contrast, we provided an upper bound on the model size required for Mamba to achieve COPY using a constructive approach in our Theorem~\ref{linear_copy}.

\textbf{Transformers and modern RNNs with CoT:} Chain-of-Thought (CoT) \citep{CoT} is employed to enhance the performance of LLMs by enabling them to provide step-by-step reasoning before arriving at a final answer. 
It has been shown theoretically that Transformers with CoT exhibit significantly improved expressive power, allowing them to solve more complex problems compared to Transformers without CoT \citep{merrill_log_precision, hedi_CoT, merrill_CoT, CoT_matengyu, hedi_linear_attention}.
Our analysis of Mamba equipped with CoT follows the framework set by \citet{hedi_CoT,hedi_linear_attention} in their analysis of dynamic programming (DP) problems.
Additionally, \citet{Lvkaifeng_rnn_not_tf}~use communication complexity to show that even with CoT, any RNN model with $o(n)$ bit memory cannot solve tasks in $T \in \{\text{Index, AR, c-gram retrieval, Counting}\}$ of size $n$ for large enough $n$, which means that the lower bound for RNNs to solve these tasks is $w(n)$.
Different from this, our work explores the ability of Mamba equipped with CoT from the perspective of solving DP problems following the setting of \citet{hedi_CoT}~and show constructions for Mamba layers with linear-scaling size relative to the sequence length to solve DP problems, which can be seen as an upper bound of the model size required to solve DP.

\section{Limitations}
In this paper, inspired by the similarity between the SSM module in Mamba and linear attention, we explore Mamba's ability to perform the COPY operation.
Furthermore, we analyze Mamba's capability in handling CoT tasks, which are modeled as DP problems.
Our findings contribute to a deeper understanding of Mamba.
However, we would like to illustrate that while Mamba may slightly underperform Transformers in certain tasks, it offers advantages in others like sparse parity learning~\cite{mamba_icl} and can achieve comparable performance with lower costs~\cite{Mamba}.
The theoretical mechanisms behind these advantages remain to be further explored.
Additionally, as shown in \citet{mamba_icl, Mamba_large_ex, Lvkaifeng_rnn_not_tf}, exploring hybrid architectures that combine the strengths of Mamba and Transformers also merits further investigation.
We leave these aspects for future work.

\bibliography{sample-base}

\newpage
\appendix

\onecolumn
\section{Appendix}

\subsection{Proof of Theorem \ref{constant_copy}} \label{app:constant_copy}

\begin{theorem}[Approximate COPY operation with constant-size SSM module]
	Given a SSM module with constant size and input sequence $\vx_1, \vx_2, \dots, \vx_{N} \in [-M, M]^{d}$ such that Assumption \ref{assum:copy} holds, then for any $\epsilon > 0$, the SSM module can approximate the COPY operation at some position $i$, that is, $\| \vy_i - \vo_i \|_{\infty} \le \epsilon$  if there is $\vc_{i}^T\vb_{pos(i)} \ge c ( \frac{1}{a_{\min}})^{L-1} + d $ where $a_{\min} = \min_{pos(i) < j \le i} a_j$ and $c, d$ are constants related to $\epsilon$, $\delta$.
\end{theorem}

\begin{proof}
	We firstly show that given the $i$-th input $\vx_i$, a SSM module can retrieve the most relevant historical record $\vv_{pos(i)}$ from the hidden state and perform the defined COPY operation under the condition illustrated in our theorem.
	To achieve this, recalling that $\vv_j = \vDelta_j \odot \vx_j$ in Eq~(\ref{eq:atten}) and denoting $\|\vDelta\|_{\infty} = \max_{i\in [N]}\|\vDelta_i\|_{\infty}$, we have that
	\begin{align}
		\left\| \vy_i - \vv_{pos(i)} \right\|_{\infty} &= \left\| \sum_{j=1}^{i}\alpha_j(\vDelta_j \odot \vx_j) \vb_j^T \vc_i - \vv_{pos(i)} \right\|_{\infty} = \left \| \sum_{j=1}^{i}\alpha_j (\vc_i^T \vb_j) \vv_j - \vv_{pos(i)}\right\|_{\infty} \\
		& = \left\| \sum_{j \neq pos(i)} \alpha_j (\vc_i^T \vb_j) \vv_j + \alpha_{pos(i)} (\vc_i^T \vb_{pos(i)}) \vv_{pos(i)}  - \vv_{pos(i)}\right \|_{\infty}  \label{copy:1}\\
		& = \left\| \sum_{j \neq pos(i)} \alpha_j (\vc_i^T \vb_j) \vv_j   + \left[ \alpha_{pos(i)} (\vc_i^T \vb_{pos(i)}) - 1 \right] \vv_{pos(i)} \right \|_{\infty} \\
		& \le M \| \vDelta \|_{\infty}  \left( \sum_{j \neq pos(i)} \alpha_j|\vc_i^T\vb_j|  +   1 - \alpha_{pos(i)} \vc_i^T \vb_{pos(i)} \right) \label{copy:2},
	\end{align}
	where in (\ref{copy:2}) we use the fact that $\| \vx_i \|_{\infty} \le M$ and $ \alpha_{pos(i)}|\vc_i^T\vb_{pos(i)}| \le 1$ (from the second condition in Assumption \ref{assum:copy}).
	Thus, to prove $\left\| \vy_i - \vv_{pos(i)} \right\|_{\infty} \le \epsilon$, we can show that
	\begin{align}\label{copy:cb1}
		\vc_i^T\vb_{pos(i)} \ge \left( 1 - \frac{\epsilon}{M \|\vDelta\|_{\infty}} \right)\frac{1}{\alpha_{pos(i)}} + \sum_{j \neq pos(i)}\frac{\alpha_j}{\alpha_{pos(i)}} \big| \vc_i^T\vb_j \big|.
	\end{align}
	Recalling that $\alpha_{j} = \prod_{k=j+1}^{i} a_k$, we have
	\begin{equation}\label{copy:a_ratio}
		\frac{\alpha_j}{\alpha_{pos(i)}} = \left\{\begin{matrix}
			\prod_{k = j+1}^{pos(i)} a_k =  a_{pos(i)} a_{pos(i)-1} \dots a_{j+1}, &  \text{when  } j<pos(i) \\
			& \\
			\prod_{k = pos(i)+1}^{j}\frac{1}{a_k} = \frac{1}{a_ja_{j-1}\dots a_{pos(i)+1}}, & \text{when  } j>pos(i).
		\end{matrix}\right.
	\end{equation}
	Then we can consider the second term on the right side of Inequality (\ref{copy:cb1}) as
	\begin{align}\label{copy:sep}
		\sum_{j \neq pos(i)}\frac{\alpha_j}{\alpha_{pos(i)}}\big|\vc_i^T\vb_j \big|& = \sum_{j \notin S_i}\frac{\alpha_j}{\alpha_{pos(i)}}\big|\vc_i^T\vb_j \big| + \sum_{j \in S_i, j \neq pos(i)} \frac{\alpha_j}{\alpha_{pos(i)}}\big|\vc_i^T\vb_j \big| 
	\end{align}
	For the first term on the right side, we have
	\begin{align}
		\sum_{j \notin S_i}\frac{\alpha_j}{\alpha_{pos(i)}}\big|\vc_i^T\vb_j \big| & \le \delta \left(  \sum_{j \notin S_i, j < pos(i)} \frac{\alpha_j}{\alpha_{pos(i)}} + \sum_{j \notin S_i, j > pos(i)} \frac{\alpha_j}{\alpha_{pos(i)}}  \right) \label{copy:second0}\\
		& \le \delta \left(\sum_{j=1}^{pos(i)-1} \frac{\alpha_j}{\alpha_{pos(i)}} + \sum_{j=pos(i)+1}^{i} \frac{\alpha_j}{\alpha_{pos(i)}} \right)  \label{copy:second1}\\
		& \le  \delta \left( \sum_{k = 1}^{pos(i)-1}(a_{\max})^k + \sum_{k=1}^{i-pos(i)}\left(\frac{1}{a_{\min}}\right)^{k} \right) \label{copy:second2}\\
		& \le  \delta \left( \frac{a_{\max}(1-a_{\max}^{pos(i)-1})}{1-a_{\max}} + \frac{\left(\frac{1}{a_{\min}}\right)^{i - pos(i)} - 1}{1- a_{\min}} \right) \label{copy:second3}
	\end{align}
	where $a_{\max} = \max_{1 \le j < pos(i)} a_j $ and $a_{\min} = \min_{pos(i) < j \le i} a_j$.
	In (\ref{copy:second0}) we use the assumption that $\big|\vc_i^T\vb_j\big| le \delta$ for $j \notin S_i$; 
	in (\ref{copy:second1}) we use $\frac{\alpha_j}{\alpha_{pos(i)}} > 0$ for all $j \le i$;
	in (\ref{copy:second2}), we use the fact that $\frac{\alpha_j}{\alpha_{pos(i)}} \le (a_{\max})^{pos(i)-j}$ for $j < pos(i)$ and $\frac{\alpha_j}{\alpha_{pos(i)}} \le \left( \frac{1}{a_{\min}} \right)^{j-pos(i)}$ for $j > pos(i)$;
	in (\ref{copy:second3}), we use the formula for the sum of a geometric series.
	
	Furthermore, considering that the vector $\vv_{pos(i)}$ to be copied must exist in the $L$-local matching set $S_i$ so there is $i - L +1 \le pos(i) \le i$, we have the following	
	\begin{align}
		\sum_{j \notin S_i} \frac{\alpha_j}{\alpha_{pos(i)}}\big|\vc_i^T\vb_j\big| & \le  \delta \left( \frac{a_{\max}(1-a_{\max}^{pos(i)-1})}{1-a_{\max}} + \frac{\left(\frac{1}{a_{\min}}\right)^{L-1} - 1}{1- a_{\min}} \right) \label{copy:second4}\\
		& \le  \delta \left( \frac{a_{\max}}{1-a_{\max}} + \frac{\left(\frac{1}{a_{\min}}\right)^{L-1}}{1- a_{\min}} \right) \label{copy:second5}.
	\end{align}
	In (\ref{copy:second4}), we use the fact that $a_{\max} \in (0,1)$ while $\frac{1}{a_{\min}} > 1$; in (\ref{copy:second5}) we just ignore the term $a_{\max}^{pos(i)-1}$ for simplicity (in fact, we can find that when $i$ is sufficiently large, the effect of this term can be neglected).
	
	Meanwhile, with Assumption \ref{assum:copy}, we can consider the second term on the right side of Inequality (\ref{copy:sep}) as
	\begin{align}
		\sum_{j\in S_i, j \neq pos(i)}\frac{\alpha_j}{\alpha_{pos(i)}} \big| \vc_i^T\vb_j \big| \le \frac{1}{\alpha_{pos(i)}} - \vc_i^T\vb_{pos(i)}      \label{copy:third}.
	\end{align}
	Thus, considering (\ref{copy:cb1}), (\ref{copy:second5}) and (\ref{copy:third}), to show $\|\vy_i - \vo_i\|_{\infty} \le \epsilon$, the following condition should be satisfied
	\begin{align}
		\vc_i^T\vb_{pos(i)} \ge& \left( 1 - \frac{\epsilon}{M \|\vDelta\|_{\infty}} \right)\frac{1}{\alpha_{pos(i)}} +\delta \left( \frac{a_{\max}}{1-a_{\max}} + \frac{\left(\frac{1}{a_{\min}}\right)^{L-1}}{1- a_{\min}} \right) \\
		&+ \frac{1}{\alpha_{pos(i)}} -  \vc_i^T\vb_{pos(i)} .
	\end{align}
	After reformulating, there exists
	\begin{align}\label{copy:final}
		\vc_i^T\vb_{pos(i)} \ge& \left( 1 - \frac{\epsilon}{ 2 M \|\vDelta\|_{\infty}} \right)\frac{1}{\alpha_{pos(i)}} + \frac{\delta}{2} \left( \frac{a_{\max}}{1-a_{\max}} + \frac{\left(\frac{1}{a_{\min}}\right)^{L-1}}{1- a_{\min}} \right) \\
		= & \frac{\delta}{2 ( 1 - a_{\min} )} \left(\frac{1}{a_{\min}}\right)^{L-1} + \left( 1 - \frac{\epsilon}{ 2 M \|\vDelta\|_{\infty}} \right) \alpha_{pos(i)}^{-1} + \frac{\delta a_{\max}}{2(1-a_{\max})} \\
		= & c \left(\frac{1}{a_{\min}}\right)^{L-1} + d, 
	\end{align}
	where we let $c = \frac{\delta}{2 ( 1 - a_{\min} )}$ and use $d$ to denote the remaining terms. 
	Thus, we complete our proof.
\end{proof}

\subsection{Proof of Theorem \ref{constant_copy_attention}}\label{app:attention_copy}
To provide the theorem, we first give the definitions and assumptions for the attention module in Transformers, which are really similar to those for the SSM module from the perspective of linear attention.
For some input sequence $\vx_1, \vx_2, \dots, \vx_{N}$, the output of the attention module we consider can be formulated as:
\begin{equation*}
	\vy_i = \sum_{j\le i} a_{i,j}\vv_{j} = \sum_{j\le i} \frac{e^{\vq_i^T \vk_j}}{\sum_{t \le i} e^{\vq_i^T \vk_t}} \vv_{j},~~ i = 1,2, \dots,N
\end{equation*}
where $a_{i,j}$ the attention scores calculated by the ${\rm Softmax(\cdot)}$ function and $\vq_i = \mW_q \vx_i$, $\vk_i = \mW_k \vx_i$, $\vv_i = \mW_v \vx_i$ are the query, key and value vectors respectively.
Then, we give the definition of COPY Operation for attention module. 
\begin{definition}[COPY Operation for the attention module]
	For a given attention module and input sequence $\vx_1, \vx_2, \dots, \vx_{N}$, the output of COPY operation is a sequence of vectors $\vo_1, \vo_2, ..., \vo_N$ with $\vo_i = \vv_{pos(i)}$ where $pos(i) \in [N]$ is the position we want to copy.
\end{definition}
In fact, this definition is very similar to the COPY for the SSM module, except that the vector being copied is changed from the historical records $\vv_i = \vDelta_i \odot \vx_i$ in the SSM module to the value vectors in the attention module directly. 
From the perspective of linear attention, the former is precisely the value vectors in attention, so the two definitions are closely related:
Similarly, we define the definition of $L,\delta$)-Matching Set for the attention module:
\begin{definition}[($L,\delta$)-Matching Set for the attention module] 
	For a given attention module and input sequence $\vx_1, \vx_2, \dots, \vx_{N}$, the ($L,\delta$)-matching set for $\vx_i$ is defined as $\gS_i(L,\delta) =  \{ j \mid |\vq_i^T \vk_j| \ge \delta, i-L < j\le i \}$.
\end{definition}
Next, we make the following assumption:
\begin{assumption}\label{assum:attention copy}
	For a given attention module and input sequence $\vx_1, \vx_2, \dots, \vx_{N}$, there exist some $L\in \sN$ and $\delta \in \sR$ such that for any $ i \in [N] $, $S_i(L,\delta)$ exists and $pos(i) \in \gS_i(L,\delta)$. In addition, for any $j \notin \gS_i(L,\delta)$, $|\vq_i^T \vk_j| < \delta$.
\end{assumption}
Note that, since the attention scores in Transformers are naturally normalized by Softmax function such that $\sum{j \in \gS_i} \le 1$, we ignore the second condition as in Assumption \ref{assum:copy} here, which would impose constraints on the stability of the output.
Then, we give our theorem as following
\begin{theorem}[Approximate COPY operation with constant-size attention module]
	Given a attention module with constant size and input sequence $\vx_1, \vx_2, \dots, \vx_{N} \in [-M, M]^{d}$ such that Assumption \ref{assum:attention copy} holds, then for any $\epsilon > 0$, the attention module can approximate the COPY operation at some position $i$, that is, $\| \vy_i - \vo_i \|_{\infty} \le \epsilon$  if there is $\vq_i^T \vk_j \ge \log\tilde{L} + c$ where $\tilde{L} = \max\{L, i - |\gS_{i}(L,\delta)|\}$ and $c$ is a constant related to $\epsilon$, $\delta$.
\end{theorem}
\begin{proof}
	For some given $i$-th input $\vx_i$, we have that 
	\begin{align}
		\left\| \vy_i - \vv_{pos(i)} \right\|_{\infty} &= \left\| \sum_{j \le i} a_{i,j} \vv_i - \vv_{pos(i)} \right\|_{\infty} = \left\| \sum_{j\le i, j \neq pos(i)} a_{i,j} \vv_i - \left(1 - a_{i,pos(i)}\right) \vv_{pos(i)}) \right\|_{\infty} \\
		& \le \left(1 - a_{i,pos(i)} +  \sum_{j \le i} a_{i,j} \right)  M \left\| \mW_V\right\|_{\infty} = 2 \left(1 - a_{i,pos(i)} \right) M \left\| \mW_V\right\|_{\infty}.
	\end{align}
	To achieve $\| \vy_i - \vo_i \| \le \epsilon$, we can show that
	\begin{align}
		a_{i,pos(i)} = \frac{e^{\vq_i^T \vk_{pos(i)}}}{\sum_{j\neq pos(i)} e^{\vq_i^T \vk_j} + e^{\vq_i^T \vk_{pos(i)}}} \ge 1-\frac{\epsilon}{2 M \left\| \mW_V\right\|_{\infty}}.
	\end{align}
	Here, we omit the condition $j \le i$ for simplicity of notation and this is equivalent to show that
	\begin{align}
		\sum_{j\neq pos(i)} e^{\vq_i^T \vk_j - \vq_i^T \vk_{pos(i)}} + 1 \le \frac{2 M \left\| \mW_V\right\|_{\infty}}{2 M \left\| \mW_V\right\|_{\infty} - \epsilon}.
	\end{align}
	Further, we have
	\begin{align}
		\vq_i^T \vk_{pos(i)} \ge \log \left[ 
		\frac{2 M \left\| \mW_V\right\|_{\infty} - \epsilon} {\epsilon} \left( \sum_{j \in \gS_i, j\neq pos(i)} e^{\vq_i^T \vk_j } +  \sum_{j \notin \gS_i} e^{\vq_i^T \vk_j }   \right)
		\right], 
	\end{align}
	where we use $\gS_i$ to denote $\gS_i(L,\delta)$ for simplicity of notation.
	Then we just need to show that $\vq_i^T \vk_{pos(i)}$ is larger than the upper bound of the right side, that is,
	\begin{align}
		\vq_i^T \vk_{pos(i)} \ge \log \left[ 
		\frac{2 M \left\| \mW_V\right\|_{\infty}} {\epsilon} \left( L e^{\rho} + (i - |\gS_i|) e^{\delta}   \right)
		\right], 
	\end{align}
	where $\rho = \max_{j \in \gS_i, j \neq pos(i)} \vq_i^T \vk_j$.
	Here we use the fact that there are at most $L$ terms in the matching set $\gS_i$ and  the assumption that $ \vq_i^T \vk_j \le \delta$ for $j \notin \gS_i$.
	By denoting $\tilde{L} = \max\{L, \gS_i\}$, we can show that,
	\begin{align}
		\vq_i^T \vk_{pos(i)} \ge \log \tilde{L} +  \log \frac{2}{\epsilon} (e^\delta + e^\rho ) M \left\| \mW_V\right\|_{\infty}. 
	\end{align}
	Thus we complete our proof.
\end{proof}

\subsection{Proof of Theorem \ref{linear_copy}}\label{app:linear_copy}
\begin{theorem}[Perform COPY operation with linear-scaling size]
	Given sequence $\vx_1, \vx_2, \dots, \vx_{N} \in [-M, M]^{d}$, there exists a SSM module with size $O(N)$ that can perform the defined COPY operation, that is, $\vy_i = \vo_i$ for any $i \in [N]$.
\end{theorem}
\begin{proof}
	Recalling that the output of SSM module in Mamba can be rewritten in the form of Eq~(\ref{eq:atten}), where  $(\vDelta_j \odot \vx_j)$, $\vb_j$, $\vc_i$ corresponds to $\vv_j$, $\vk_j$ and $\vq_i$ respectively. 
	Our intuition is to store all the information of $\vv_i$ from our history in the hidden state space of size $O(N)$ (similar to the KV cache in attention format), and then use the appropriate $ \vc_i $ as the query for retrieval. 
	We can set $ \mA = \mO $ so that Eq~(\ref{eq:atten}) further transforms in a way that does not forget historical information, that is, $\vy_i = \sum_{j=1}^{i}(\vDelta_j \odot \vx_j) \vb_j^T \vc_i = \sum_{j=1}^{i}\vv_j \vb_j^T \vc_i$.
	
	Let $\tilde{\vx}_i = [\vx_i, \ve_i, \ve_{pos(i)}] \in \sR^{d+2N}$ where $\ve_i \in \sR^N$ denote the one-hot vector where only the $i$-th value is 1.
	We use $\ve_i$, $\ve_{pos(i)}$ to denote the current position and the position of historical token we want to copy respectively. 
	Then, we construct $\mW_\vb = [\mO_{N \times d}, \mI_{N}, \mO_N] \in \sR^{N\times(d+2N)}$ so that $\tilde{\vb}_i = \mW_\vb \tilde{\vx}_i = \ve_i$.
	Then, at the $i$-th step, the information newly recorded in the state space will be $\vv_i \tilde{\vb}_i^T = \vv_i \ve_i^T \in \sR^{d \times N}$ and the updated state space will be $ \mH_i = \mH_{i-1} + \sum_{j=1}^{i-1}\vv_j \tilde{\vb}_j^T + \vv_i \tilde{\vb}_i^T = [ \vv_1, \vv_2, \dots,  \vv_i, \mO_{d \times (N-i)}]$ thus at the last step, we can record all historical information in the state space by $\mH_T = \sum_{j=1}^{N}\vv_j \tilde{\vb}_j^T = [\vv_1, \vv_2, \dots, \vv_N]$.
	Then, at the output process, we can construct $\mW_\vc = [\mO_{N \times d}, \mO_N, \mI_{N}] \in \sR^{N\times(d+2N)}$ so that $\tilde{\vc}_i = \mW_\vc \tilde{\vx}_i = \ve_{pos(i)}$.
	Thus, the output will be $\vy_i = \mH_i \tilde{\vc}_i =  \sum_{j=1}^{i}\vv_j \vb_j^T \ve_{pos(i)} = \vv_{pos(i)}$.
	
	At the same time, we note that in the above process, the vectors $\ve_i, \ve_{pos(i)}$ to denote position in $\tilde{\vx}_i$ are sparse. 
	In fact, we only need to use two indices $p_i = i$ and $p_{pos(i)} = pos(i)$ to store them thus the total size to store all indices is $O(N)$.
	Additionally, $\mW_\vb, \mW_\vc$ are also sparse so we require at most $O(N)$ space to store these two matrices.
	Therefore, the model size we need is $ O(Nd) $ that scales linearly with the length $N$.
	Thus, we complete our proof.
\end{proof}
\begin{remark}
	It should be noted that we can also degenerate a linear-size Mamba block into the aforementioned SSM module by deactivating the gating branch to achieve the COPY operation.
	For example, we can set $\mW_1 = \mW_{3} = \mI, \vb_1 = \vzero, \mW_2 = \mO$ and $\vb_2 = k\vone$ where the constant $k$ satisfies $\sigma(k) = 1$.
	In addition, we note that $\vx_{i} \in [-M,M]^d$ and $p_i, p_{pos(i)} \in [1,N]$ thus the largest value involved in the aforementioned process will not exceed $NM^2$ (the largest value in hidden states), which is upper bounded by $O(poly(M,N))$.
	All parameters being upper bounded by $O(poly(M,N))$ means that the problem can also be solved by the same Mamba block with $log(N)$ precision, which has been also adopted in previous works \cite{merrill_log_precision,hedi_CoT,Lvkaifeng_rnn_not_tf}.
\end{remark}

\subsection{Examples for DP Problems}\label{app:dpexample}
We consider the problem of finding the minimum edit distance, where we aim to transform string $s^{(1)}$ into string $s^{(2)}$ with lengths $n_1 = |s^{(1)}|$ and $n_2 = |s^{(2)}|$, respectively. 
The costs for insertion, deletion, and substitution are $a$, $b$, and $c$, respectively. 
Obviously, the input sequences for this problem are $\{ s^{(1)}, s^{(2)}\}$ and the scale of the problem is $\vn = [n_1, n_2]^T$.
Further, the state space is $\gI_{\vn} = \{ (i,j) | 1 \le i \le n_1, 1 \le j \le n_2  \}$. 
Let $\rm dp(j,k)$ represent the cost of transforming the first $j$ characters of $s_1$ into the first $k$ characters of $s_2$. 
The transition function can then be expressed as:  
$$
\text{dp}(j, k) =\begin{cases}    ak & \text{if } j = 0 \\    bj & \text{if } k = 0 \\    \min \big(         \text{dp}(j, k-1) + a, \text{dp}(j-1, k) + b,\\        \qquad \text{dp}(j-1, k-1) + c \mathbb{I}[s_j^{(1)} \neq s_k^{(2)}]    \big) & \text{otherwise}\end{cases}
$$
Finally, the aggregation function selects $\rm dp(n_1, n_2)$ as the final answer.  In the example above, the size of the state space, all intermediate values, and the lengths of the input strings will all be upper-bounded by $\rm poly(n_1, n_2)$. Moreover, the operations required by the involved functions can be approximated with polynomial efficiency by a constant-size MLP, as shown in Lemmas in Appendix \ref{app:cot}. Therefore, such DP problems satisfy Assumption \ref{assum:dp}.
It can be referenced from Section 4.1 of \citet{hedi_CoT} for more detailed examples for DP problems.

\subsection{Explanation of $m$-locality:}\label{app:m-locality} 
The most common arithmetic tasks satisfy the $m$-locality assumption, which are also frequently encountered in current LLM applications. Consider an arithmetic expression of length $n$; the CoT reasoning required to compute it step by step may have a total output length of $T = O(n^2)$. However, the window size $m$ required for computing each intermediate state does not exceed $n$.
For example, consider the expression $7 * (7 + 5 + 10) =$. The input length is $n = 10$, and its complete CoT output would be:
$7 * ( 7 + 5 + 10 ) = 7 * ( 1 + 10 ) = 7 * 0 = 0$.
As we can see, aside from the input part, the generation of each subsequent token requires a context window no larger than the input length $n$. Therefore, this task satisfies the $m$-locality.

Another example that satisfies the $m$-locality is the Edit Distance problem. When computing the 2D DP matrix, the number of entries (or tokens) required for each step does not exceed $m \le \max\{n_1, n_2\}$. For instance: x g v $|$ x g o  \texttt{<p>} 0 2 4 ,  2 0 2 ,  4 2 3 , \texttt{<p>} 3 where \texttt{<p>} denotes the separate token to enclose the DP matrix arranged row by row.
If we copy the input sequence when computing each row of the DP matrix, such as:
\begin{equation*}
	\begin{aligned}
		&\text{x g v}~|~\text{x g o}~\texttt{<p>}~\text{x g v}~|~\text{x g o}~0~2~4~, \text{x g v}~|~\text{x g o}~\texttt{<p>}~\\
		&2~0~2~,~\text{x g v}~|~\text{x g o}~\texttt{<p>}~4~2~3~,\texttt{<p>}~3	
	\end{aligned}
\end{equation*}
then the entire process becomes strictly $m$-locality, where $m \le 2n + 2$ with $n = |s_1| + |s_2|$ representing the input size. Meanwhile, the total CoT output length is roughly $O(n^2)$.

However, for the LIS problem, computing each state might require revisiting the entire input string, and the corresponding CoT output length itself is $T = O(n)$.
This does not satisfy $m$-locality because $m$ is at least $n$ and cannot be significantly smaller than the output length $T$ in terms of order.

\subsection{Proof of Theorem 3}\label{app:cot}
In this part, we first present the necessary lemmas before completing the proof of Theorem~\ref{DP_Mamba}. 
In fact, these lemmas are very similar to those presented by \citet{hedi_CoT} in Appendix C.1 regarding MLP.
The main difference is that we need to degenerate Mamba blocks to MLP and consider different activation functions (SiLU for Mamba instead of GELU). 
Thus, we only provide detailed proofs of these relevant lemmas when necessary.
\begin{lemma}[Perform multiplication]\label{lemma:multiply}
	For any $\epsilon > 0$ and $M > 0$, there exists Mamba block parameters with $l_{\infty}$ norm upper bounded by $O(poly(M, 1/\epsilon))$ such that $| f(a,b) -  ab| \le \epsilon $ holds for all $a, b \in [-M, M]$.
\end{lemma}
\begin{proof}
	We first show that a two-layer MLP using the SiLU activation function can achieve the above operation. We use the same construction as in Lemma C.1. in \citet{hedi_CoT}, except that we use the SiLU activation function instead of GELU. 
	Specifically, let $g: \sR^2 \rightarrow \sR$ be a two-layer MLP with SiLU activation, and the hidden dimension is 4, then we can construct $f$ as
	\begin{equation}\label{multiply_construct}
		g(a,b) = \frac{\lambda^2}{2} \left( \sigma{\left(\frac{a+b}{\lambda} \right)} + \sigma{\left(\frac{-a-b}{\lambda} \right)} - \sigma{\left(\frac{a-b}{\lambda} \right)} - \sigma{\left(\frac{-a+b}{\lambda} \right)}   \right),
	\end{equation}
	where $\lambda$ is a scaling factor. 
	In addition, considering $\sigma(x) = \frac{x}{1+e^{-x}}$, $\sigma'(x) = \frac{1+(x+1)e^{-x}}{(1+e^{-x})^2}$, $\sigma''(x) = \frac{e^{-x}(2+2e^{-x} + xe^{-x}- x)}{(1+e^{-x})^3}$, we have $\sigma(0) = 0$, $\sigma'(0) = \frac{1}{2}$, $\sigma''(0) = \frac{1}{2}$. 
	Then, using the Taylor expansion with the Lagrange remainder, we can obtain that 
	\begin{align*}
		& \sigma{\left(\frac{a+b}{\lambda} \right)} + \sigma{\left(\frac{-a-b}{\lambda} \right)} - \sigma{\left(\frac{a-b}{\lambda} \right)} - \sigma{\left(\frac{-a+b}{\lambda} \right)}  \\ 
		= & ~ \frac{1}{2!}\frac{1}{2} \left( \left( \frac{a+b}{\lambda} \right)^2 + \left( \frac{-a-b}{\lambda} \right)^2 - \left( \frac{a-b}{\lambda} \right)^2 - \left( \frac{-a+b}{\lambda} \right)^2 \right) + R_2 =  ~\frac{2ab}{\lambda^2} + R_2,
	\end{align*}
	where $R_2$ is the second-order remainder term. 
	Assuming that $\lambda > 2M$,  we have $ | \frac{\pm a \pm b}{\lambda}| < \frac{2M}{\lambda} < 1$ and then
	\begin{align*}
		|R_2| &\le \frac{4}{3!} \left( \frac{2M}{\lambda}\right)^2 \max_{x \in [-1, 1]}|\sigma'''(x)| \\
		&=  \frac{4}{3!} \left( \frac{2M}{\lambda}\right)^2 \max_{x \in [-1, 1]}\biggl|\frac{(x-3)e^{-x}-4xe^{-2} + (x+3)e^{-3x}}{(1+e^{-x})^4} \biggl| \\
		&\le \frac{4}{3!} \left( \frac{2M}{\lambda}\right)^2 \frac{4e + 4e^2 + 4e^3}{(1+e^{-1})^4} \\
		&\le \frac{4}{3!}\frac{8M^3}{\lambda^3} \frac{81}{2} \\
		&= \frac{216M^3}{\lambda^3}.
	\end{align*}
	Thus if we set $\lambda \ge \frac{216M^3}{2\epsilon}$  we will have $|g(a,b) - ab|  \le \frac{\lambda^2}{2} |R_2| \le \epsilon$.
	
	Then, we note that a Mamba block $\vf$ defined as Eq~(\ref{mblock}) can degenerate into the above MLP $g$ by deactivating its SSM branch. 
	Specifically, we only need to set $\mW_1$ to be zeros and $\vb = \vone $ so that the input of the SSM branch is a constant 1, that is, 
	\begin{equation*}
		\vf(\vx) = \mW_3 \cdot \mathrm{SSM}(\vone) \odot \sigma(\mW_2\vx + \vb_2).
	\end{equation*}	
	In the SSM module, we can set $\mW_\vb$ to be zeros, that is, no new information will be retained in the hidden state.
	Following this, we set $\vd = \vone$ resulting that given $\vx = \vone$,  we have $\mathrm{SSM}(\vone) = \vy = \vc^T\mH + \vd\odot \vx = \vc^T \vone + \vone \odot \vone = \vone$. 
	Thus the SSM branch can be deactivated and the Mamba block will degenerate into a two layer MLPs, that is, 
	\begin{equation}\label{mlp}
		\vf(\vx) = \mW_3  \sigma(\mW_2\vx + \vb_2).
	\end{equation}
	Furthermore, given $\vx = [a,b]$, we can set $\mW_2 \in \sR^{4 \times 2}$ and $\mW_3 \in \sR^{4 \times 1}$ to meet the two-layer MLP $g$ as Eq~(\ref{multiply_construct}). 
	Additionally, we note that all parameters of this Mamba block can be upper bounded by $O(poly(M, 1/\epsilon))$ under the $l_{\infty}$ norm.
	Thus, we complete our proof.
\end{proof}
\begin{remark}
	It should be noted that here we have only provided one possible construction and this is not unique. For example, in the process of deactivating the SSM branch, we could also choose to make $\vDelta$ sufficiently large and correspondingly $\tilde{\mA}$  sufficently small with $\mA \le \vzero$ so that the hidden states approximates zeros. 
	In fact, the expressive power of an Mamba block with two branches should be stronger than that of a two-layer MLP since it already encompasses the latter. Nevertheless, we still provide one possible construction here.
\end{remark}

\begin{lemma}[Approximate two-layer MLPs with ReLU]\label{lemma:mlp}
	Let $\vg: \sR^{d_1} \rightarrow \sR^{d_2}$ be a two-layer MLP with ReLU activation, and all parameters are upper bounded by $M$. 
	Then, for any $\epsilon > 0$ , there exists a Mamba block $\vf$ and parameters upper bounded by $O(poly(M,1/\epsilon))$ in the $l_{\infty}$ norm, such that for all $\vx \in \sR^{d_1}$, we have $\|\vf(\vx) - \vg(\vx) \|_{\infty} \le \epsilon$. 
\end{lemma}

\begin{proof}
	Similar to Lemma \ref{lemma:multiply}, once again, we deactivate the SSM branch, causing a Mamba block to degenerate into the form of Eq~\ref{mlp}. Considering a two-layer MLP with a ReLU activation function denoted as $g(\vx) = \overline{\mW}_{3}\mathrm{ReLU}(\overline{\mW}_2\vx)$ where $\overline{\mW}_2 \in \sR^{d \times d_1}$ and  $\overline{\mW}_3 \in \sR^{d_2 \times d}$, we can set similar parameters for the degenerated Mamba blcok, that is, we consider $\mW_2 = \lambda \overline{\mW}_2$, $\mW_{3} = \frac{1}{\lambda}\overline{\mW}_3$ in Eq~(\ref{mlp}) where $\lambda$ is some large constant.
	In order to prove the lemma, we need to show that $\| \vf(\vx) - \vg(\vx) \|_{\infty} \le \epsilon$ with some $\lambda$ upper bounded by $O(poly(M,1/\epsilon))$.
	
	Considering a scalar $z \in \sR$, we firstly consider the upper bound of the following equation:
	\begin{equation*}
		\biggl| \relu(z) - \frac{1}{\lambda} \silu(\lambda z) \biggl|  = \biggl| \max(z,0) - \frac{z}{1+e^{-\lambda z}} \biggl| = \frac{|z|}{e^{\lambda |z| } + 1} \le \frac{1}{\lambda},
	\end{equation*}  
	where we use the fact that $e^{x} + 1 > x $ for any $x \ge 0$. Then, let $\vz = \overline{\mW}_2\vx$, we can show that for any $\vz \in \sR^{d}$,
	\begin{align}
		\left \|\overline{\mW}_3 \mathrm{ReLU}(\vz) -  \frac{1}{\lambda} \overline{\mW}_3 \silu (\lambda \vz)   \right\|_{\infty} & \le \|\overline{\mW}_3\|_{\infty} \left\| \relu(\vz) - \frac{1}{\lambda}\silu(\lambda \vz)\right \|_{\infty} \\
		& \le Md \left\| \relu(\vz) - \frac{1}{\lambda}\silu(\lambda \vz)\right \|_{\infty} \\
		& \le Md \max_{z\in \sR} \biggl| \relu(z) - \frac{1}{\lambda} \silu(\lambda z) \biggl| \\
		& \le \frac{Md}{\lambda}.
	\end{align}
	Then, if we set $\lambda > \frac{Md}{\epsilon}$, we will have $\| \vf(\vx) - \vg(\vx) \|_{\infty} \le \epsilon$ and all parameters of the Mamba block is upper bounded by $O(poly(M,1/\epsilon))$.
	Thus, we complete our proof.
\end{proof}
\begin{remark}
	We have proven that a Mamba block can approximate a two-layer MLP with ReLU activation function, and since the latter can perform many basic operations, including linear transformations and selection operations as constructed in Lemma C.3 and Lemma C.5 in \citet{hedi_CoT}, we can use Lemma \ref{lemma:mlp} to adopt the same construction, enabling the Mamba block to perform these operations. We present the following colloary more specifically, and the detailed proof can be found in the above mentioned part in \citet{hedi_CoT}.
\end{remark}

\begin{lemma}[Perform linear transformation, easily derived from Lemma \ref{lemma:mlp} and Lemma C.3 in \citet{hedi_CoT}]\label{lemma:linear}
	Let $\mW \in \sR^{d_2 \times d_1}$ be any matrix used for implementing linear transformations upper bounded by $M$ and $\vf: \sR^{d_1} \rightarrow \sR^{d_2}$ be a Mamba block.
	Then, for any $\epsilon > 0$, there exist Mamba block parameters with $l_{\infty}$ norm bounded by $O(poly(M,1/\epsilon))$, such that for any $\vx \in \sR^{d_1}$, we have $\| \vf(\vx) - \mW\vx\|_{\infty} \le \epsilon$.
\end{lemma}

\begin{lemma}[Perform select operation, easily derived from Lemma \ref{lemma:mlp} and Lemma C.4 in \citet{hedi_CoT}] \label{lemma:select}
	Define the selection function $\vg: \sR^d \times \sR^d \times \sR \rightarrow \sR^d$ as follows:
	\begin{equation}
		g(\vx, \vy, t) = \left\{\begin{matrix}
			\vx & if~t > 0\\
			\vy & if~t < 0
		\end{matrix}\right.
	\end{equation}
	Let $\vf: \sR^{d} \times \sR^{d} \times \sR \rightarrow \sR^{d}$ be a Mamba block. Then, for any 
	$\epsilon > 0$, $\alpha > 0 $, and $M > 0$, there exist Mamba parameters with $l_{\infty}$ norm bounded by $O(poly(M, 1/\alpha, 1/\epsilon))$, such that for all $\vx \in [-M, M]^d$, $\vy \in [-M, M]^d$, and $t \in [-\infty, -\alpha] \cup [\alpha, +\infty]$, we have $\|\vf(\vx, \vy, t) - g(\vx, \vy, t) \|_{\infty} \le \epsilon $.
\end{lemma}

Next, we show that one Mamba layer or several Mamba layers can implement indicator functions through the select operation.
We mainly focus on the usual indicator functions $\sI[a\neq b]$, $\sI[a>b]$ and $\sI[a<b]$.

\begin{lemma}[Perform indicator function]\label{lemma:indicator}
	Define the indicator function $\sI(a, b, \circ): \sR^2 \times \{\neq, > , <\} \rightarrow \{0,1\}$ where $a,b \in [-M, M]$. The output of the function will be $1$ if $a \circ b$ is satisfied otherwise the output will be $0$. Let $f: \sR^2 \rightarrow \sR$ be a Mamba block. Then, for any $\epsilon > 0 $, there exist Mamba parameters with $l_{\infty}$ norm upper bounded by $O(poly(M,1/\epsilon))$, such that for any $a,b \in [-M, M]$ and $\circ \in \{ \neq, >, < \}$, we have $\| f(a,b) - \sI(a,b,\circ)\|_{\infty} \le \epsilon$.
\end{lemma}

\begin{proof}
	We first show that a Mamba block can implement $\sI[a>b]$ and $\sI[a<b]$.
	For $\sI[a>b]$, it is equivalent to consider $g(1,0,a-b)$ where $g(\cdot)$ defined in Lemma~\ref{lemma:select}.
	So firstly we can use a linear layer with appropriate parameters $\mW_0, \vb_0$ to convert the input $[a,b]$ into the vector $[1,0,a-b]$.
	Then we can use Lemma~\ref{lemma:select} to implement $\sI[a>b]$ by changing the parameters of the first linear layer from $\{ \mW_1, \vb_1\}$ to $\{ \mW_1\mW_0, \vb_1 + \mW_1\vb_0\}$.
	The proof for $\sI[a < b]$ is similar as well.
	
	Noticing that $\sI[a \neq b] = 1 - (1 - \sI[a>b]) \cdot (1 - \sI[a < b])$, we can implement $\sI[a \neq b]$ through the following layers:
	Firstly, we can use one Mamba block to implement $1 - \sI[a>b]$ and $1 - \sI[a<b]$ simultaneously, where the hidden dimension will be $8$ and the output is a vector $[1 - \sI[a>b], 1 - \sI[a<b]]$.
	Then, another Mamba block is constructed to implement the multiplication $(1 - \sI[a>b]) \cdot (1 - \sI[a < b])$ according to Lemma~\ref{lemma:multiply} and the appropriate outermost linear layer parameters are chosen to simultaneously achieve multiplication by a negative sign and addition of a bias of 1, where the hidden dimension will be $4$ and the output will be $\sI[a \neq b]$.
	Thus, we complete our proof.
\end{proof}

Furthermore, following the setting of \citet{hedi_CoT}, we also make the following assumption:
\begin{assumption}\label{assum:dp}
	Given input sequences $s^{(1)}, s^{(2)}, \dots, s^{(N)}$, we consider the following constraints for the DP problem:
	\begin{itemize}
		\item For any $i \in \gI_\vn$, there exists constants $N_{\vs}$, $N_{\rm dp}$ and $N_{\gA}$ such that $|\gI_i| \le N_{\vs}$, $|\gV_{{\rm dp}(i)}| \le N_{\rm dp}$ and $|\gA_{\vn}| \le N_{\gA}$.
		\item The size of the state space $|\gI_\vn|$, the embeddings of all tokens in the input sequence, all intermediate DP values $({\rm dp}(i)$ for $i \in \gI_{\vn})$,  and the final answer $A$ can all be polynomially upper bounded by the problem size $\vn$.
		\item The functions used to solve the DP problem, including the function $f_{\gI}$ to determine the next state, the transition function $f_{\gT}$, the aggregation function $f_{\gA}$ and $\gA(\vn)$ can all be approximated with polynomial efficiency by a constant-size MLP (with the SiLU activation function).
	\end{itemize}
\end{assumption}
\begin{remark}
	The first constraint of the Assumption \ref{assum:dp} illustrates that the number of input tokens and previous DP values used in the transition function at each step can be upper bounded by $ N_{\vs}$ and $N_{\rm dp}$. 
	In addition, the number of DP values used in aggregation is at most $N_{\gA}$.
	This is reasonable because the number of inputs for solving each state in a DP problem should be finite.
	The second constraint is a restriction on the magnitude of the intermediate values, allows that all involved inputs and outputs used in functions can be represented by the log-precision model.
	The third constraint allows a constant-sized degenerated Mamba to implement functions required to solve the DP, which has been proved by above Lemmas~\ref{lemma:multiply}-\ref{lemma:indicator}. 
	In fact, due to the first constraint, the sizes of inputs and outputs of these functions will be a constant related to $\{N_{\vs}, N_{\rm dp}, N_{\gA}\}$.
	
	In fact, Assumption 2 covers many dynamic programming (DP) problems commonly encountered in real-world scenarios, such as basic arithmetic operations, the Longest Increasing Subsequence (LIS), and Edit Distance (ED). The sizes of these tasks grow polynomially with respect to the input size. While the functions involved in solving subproblems of these tasks may be non-smooth, we have shown in Lemmas~\ref{lemma:multiply}-\ref{lemma:indicator} that they can all be approximated by MLPs with polynomial efficiency. Therefore, Assumption 2 is reasonable for DP problems.
\end{remark}

Now, based on the basic operations that can be implemented by the Mamba blocks as discussed above, we present the proof of Theorem~\ref{DP_Mamba}:
\begin{theorem}[Perform DP problems with CoT]
	Considering any DP problem and given input sequences that satisfies Assumption~\ref{assum:dp}, for any integer $T \in \sN$, there exists several Mamba layers with size $O(T)$, such that the answer generated by the Mamba layers will be correct when the length of the answer is no more than $T$.
\end{theorem}

\begin{proof}
	Firstly, we illustrate the input format for the DP problem. We follow the embedding format in the proof of Theorem 4.7 in \citet{hedi_CoT}, that is, assuming that the input at any step of solving the DP problem using CoT is a sequence of tokens  embedded as follows:
	\begin{equation*}
		\vx_{t}^{(0)} = \left[  \ve_{t}^{\rm input}, \ve_{t}^{\rm state}, \ve_{t}^{\rm dp}, \ve_{t}^{\rm answer}, \ve_{t}^{\rm sep}, t, 1 \right],
	\end{equation*}
	where the specific value of each part is depend on the content represented by the current token. More specifically, each part can be described as:
	\begin{itemize}
		\item If the current position denotes a input token, then we set $e^{\rm input}_t$ as the embedding of the input and simultaneously set $\ve_{t}^{\rm state} = \ve_{t}^{\rm dp} = \ve_{t}^{\rm answer} = \ve_{t}^{\rm sep} = \vzero$.
		\item If the current position is the final answer, then $\ve_t^{\rm anwser}$ denotes the embedding of the answer and we set  $\ve_{t}^{\rm input} = \ve_{t}^{\rm state} = \ve_{t}^{\rm dp} =  \ve_{t}^{\rm sep} = \vzero$.
		\item If the current position denotes the $j$-th separator $|$ between input sequences , then we set $\ve_{t}^{\rm sep} = \ve_j$ and  $\ve_{t}^{\rm input} = \ve_{t}^{\rm state} = \ve_{t}^{\rm dp} =  \ve_{t}^{\rm answer} = \vzero$.
		\item If the current position denotes an intermediate DP state, then we use $\ve_t^{\rm state}$ to denote the embedding of the DP state and $\ve_t^{\rm dp}$ denotes the corresponding value.
		Similarly, other part will be set to be $\vzero$.
		\item The scalar $t$ denotes the current position in the whole sequence, which holds the value for all above cases.
	\end{itemize} 
	
	We illustrate that here we use a concatenation operation to replace the residual connection definced in Eq~(\ref{mlayer}), which is a technique also used by \citet{hedi_CoT,hedi_linear_attention} in similar proofs. 
	This is because, from the perspective of expressive capability, the two operations are equivalent: the output of a Mamba block $\vy = \vf(\vx)$ concatenated with the input (that is, $[\vy, \vx]$) can also be represented using the residual connection: $\vg([\vx, \vzero]) + [\vzero, \vx] = [\vy, \vzero] + [\vzero, \vx] = [\vy, \vx]$ where $\vg: \sR^{2d} \rightarrow \sR^{2d}$ is another Mamba block and part of its parameters to be same as $\vf$ and others are set to be $\vzero$.
	Conversely, the  concatenation can implement residual connection by using a linear projection.
	
	Here, we show our construction of several Mamba layers to solve the DP problem, which is composed of different blocks to perform different tasks:

	\noindent \textbf{Block 1:} The first block aims to calculate the problem size $\vn$ and the embedding of the next state $\ve^{\rm next\_state}_t$.  
	This process can be described as follows:
	
	\begin{itemize}
		\item \textbf{Compute the problem size $\vn$:}
		(i) First, we can replicate the position of the token $t_{sep, 1}, t_{sep,2}, \dots, t_{sep, N}$ using the COPY operation. This can be achieved with a Mamba layer of size $O(Ntd)$ according to Theorem~\ref{linear_copy};
		(ii) Then, we calculate the size of the problem as $\vn = [t_{\rm sep, 1} - 1, t_{\rm sep, 2} - t_{\rm sep, 1} - 1, \dots, t_{\rm sep, N} - t_{\rm sep, N-1} -1]$, which can be done by applying a linear transformation using one Mamba layer, as shown by Lemma~\ref{lemma:linear}.
		
		\item \textbf{Obtain the next state $\ve^{\rm next\_state}$:}
		According to Assumption \ref{assum:dp}, the function $\ve^{\rm next\_state} = f(\vn, \ve^{\rm state})$ which determines the next state, can be approximated by constant-sized MLPs. 
		Thus, this can also be implemented by having several Mamba layers degenerate into MLPs.
	\end{itemize}
	The output after this step can be written as:
	\begin{equation*}
		\vx_{t}^{(1)} = [ \ve_{t}^{\rm input}, \ve_{t}^{\rm state}, \ve^{\rm next\_state}, \ve_{t}^{\rm dp}, \ve_{t}^{\rm answer}, \ve_{t}^{\rm sep}, \vn, t, 1  ]
	\end{equation*}
	
	\noindent \textbf{Block 2:} The second block is mainly constructed to find the indices of input tokens and intermediate DP values that are needed to calculate the DP value corresponding to $\ve^{\rm next\_state}$. Specifically, this can be described as follows: 
	\begin{itemize}
		\item \textbf{Calculate the needed indices:} We calculate the positions of the input token $\vp^{\vs}_t =  I_{\vs}(\vn, \ve^{\rm state})$ and the positions of tokens that correspond to needed DP values $\vp^{\rm dp}_t = I_{\rm dp}(\vn,  \ve^{\rm state})$. If $I_{\vs}(\vn, \ve^{\rm state}) = \emptyset$ or $I_{\rm dp}(\vn,  \ve^{\rm state}) = \emptyset$, we set the positions to be a special value $\gamma$.
		According to Assumption~\ref{assum:dp}, these two functions can be done by constant-size MLPs thus can be approximated by degenerated Mamba layers. 
		
		\item \textbf{Set the flag:} 
		(i) Set the flag $f^{\rm answer}_t$ based on whether the DP value of the current state is needed in the final aggregation function. This can be achieved by several Mamba layers with Assumption \ref{assum:dp} that the function $\gA = f(\vn, \vs)$ can be approximated by MLPs and additionally using Lemma~\ref{lemma:indicator} to implement $\gI[\ve^{\rm state}_t \neq \ve^{\rm state}_j]$ where $\ve^{\rm state}_j \in \gA$.
		(ii) Set the flag $f^{\rm state}_t$ to denote whether the current state is the last state. This can be implemented by checking $\sI[\ve^{\rm next\_state}_t \neq \vzero]$ with Mamba layers using Lemma~\ref{lemma:indicator}.
		
	\end{itemize}
	The output result after this step can be written as:
	\begin{equation*}
		\vx_{t}^{(2)} = [ \ve_{t}^{\rm input}, \ve_{t}^{\rm state}, \ve^{\rm next\_state}, \ve_{t}^{\rm dp}, \ve_{t}^{\rm answer}, \ve_{t}^{\rm sep}, \vn, \vp^{\vs}_t, \vp^{\rm dp}_t, f^{\rm answer}_t, f^{\rm state}_t,  t, 1  ]
	\end{equation*}
	
	\noindent \textbf{Block 3:} This block is designed to calculate the DP value for the next state. In detail, the implementation involves the following steps:
	\begin{itemize}
		\item \textbf{Check the flag:}
		We check the flag $f_t^{\rm state}$ using several Mamba layers using \ref{lemma:indicator} to implement $\sI[f_t^{\rm state} \neq 1]$ using Lemma~\ref{lemma:indicator}. If $f_t^{\rm state} = 1$, the current denote is the last state and we just need to set $\vp_t = \gamma \vone$ where $\vp_t$ denotes $\vp_t^{\vs}$ and $\vp_t^{\rm dp}$, which implies $I_{\vs}(\vn, \ve^{\rm state}) = \emptyset$ and $I_{\rm dp}(\vn,  \ve^{\rm state}) = \emptyset$, that is, no input tokens or DP values are needed.
		\item \textbf{Obtain the needed embeddings:} 
		If $f_t^{\rm state} \neq 1$ , then (i) We copy the input token embeddings $\ve^{\rm input}$ at positions $\vp_t^{\vs}$. This COPY operation can be implemented by a Mamba layer of size $O(N_std)$ using Theorem~\ref{linear_copy}; 
		(ii) Simultaneously, we copy the embeddings  of DP values at positions $\vp_t^{\rm dp}$, which can be achieved by a Mamba layer of size $O(N_{\rm dp} td)$. If the position is empty, we just need to check $\sI[\vp_t \neq \gamma \vone]$ and set the needed embeddings $\ve^{\rm input}$ or $\ve^{\rm dp}$ to be some special token. Totally, the size of Mamba layers in this step is $O((N_s+N_{\rm dp})td)$.
		\item \textbf{Calculate the DP value:} We calculate the DP value $\ve_{t}^{\rm next\_state}$ for the next state with the Assumption \ref{assum:dp} that the transition function can be approximated by several Mamba layers using Lemma~\ref{lemma:mlp}.
	\end{itemize}
	The output result after this step can be written as:
	\begin{equation*}
		\vx_{t}^{(2)} = [ \ve_{t}^{\rm input}, \ve^{\rm next\_state}, \ve_{t}^{\rm next\_dp}, \ve_{t}^{\rm answer}, \ve_{t}^{\rm sep}, \vn, f^{\rm answer}_t, f^{\rm state}_t,  t, 1  ]
	\end{equation*}
	
	\noindent \textbf{Block 4:} The last block is constructed to implement the final aggregation function and output the final answer. Specifically, the steps are as follows:
	\begin{itemize}
		\item \textbf{Check the flag:} We identify whether the current state is the last state by checking $\sI[f^{\rm state}_t \neq 1]$ by using Lemma~\ref{lemma:indicator}.
		If $f^{\rm state}_t = 1$, then all intermediate DP values have been solved and we need to compute the final answer. 
		\item \textbf{Obtain the needed embeddings:} We collect the DP value embeddings $\ve^{\rm dp}$ of these tokens whose $f^{\rm answer} = 1$, which can be achieved by COPY operation according to Theorem~\ref{linear_copy} with one Mamba layer of size $O(N_{\gA}td)$.
		\item \textbf{Generate the final answer:} Finally, we compute the answer by implementing the aggregation function, which can be achieved by constant-size MLPs according to Assumption~\ref{assum:dp}, thus can also be achieved by several degenerated Mamba layers.
	\end{itemize}
	
	In summary, given a sequence length $t$ and  equipped with CoT, the parameter size required by the Mamba layers to generate the correct answer at each step is $O(\tilde{N}td)$, where $\tilde{N} = \max \{ N, N_{\vs} + N_{\rm dp}, N_{\gA} \}$ is a constant independent of $t$, that is, the size of the Mamba layer scales linearly with $t$.
	Thus, we complete our proof. 
\end{proof}

\subsection{Proof of Theorem~\ref{localDP_Mamba}}\label{app:cot_local}
\begin{theorem}[Perform $m$-locality DP problems with CoT]
	Consider any $m$-locality DP problem and given input sequences that satisfies Assumption~\ref{assum:dp}, for any integer $T \in \sN$, there exists several Mamba layers with size $O(m)$, such that the answer generated by the Mamba layers will be correct when the length of the answer is no more than $T$.
\end{theorem}
\begin{proof}
	The overall proof construction approach is similar to that of Theorem~\ref{DP_Mamba}, with the only difference being that under the assumption of $m$-locality, when performing the COPY operation, the constructed Mamba only needs to focus on at most $m$ tokens preceding the current position. This results in the size of the Mamba layers only needing to be $O(\tilde{N}md)$.
\end{proof}

\subsection{More Details of Experiments}\label{app:ex}

For the copy task experiments, we mainly refer to the setup by \citet{phonebook}. For Transformers, we select the GPT-NeoX architecture \cite{gpt-neox-library} while for Mamba we use the Mamba GitHub repository\cite{Mamba}. 
More specifically, for the left part of Figure \ref{fig:ex}, we configure 10 layers for the Transformer (TF-126M) and 5 layers for TF-63M, with both having a hidden size of 1024 and RoPE \cite{rope} as the positional encoding. 
For Mamba models, we configured 20 layers and a hidden size of 1024 for Mamba-135M, 20 layers and a hidden size of 720 for Mamba-D-67M, 10 layers and a hidden size of 1024 for Mamba-L-67M, and 40 layers and a hidden size of 512 for Mamba-LD-69M.
We use an online sampling batch size of 8 and set the maximum context length to 220, meaning each example often contains multiple instances. 
AdamW\cite{AdamW} is chosen as the optimizer with a learning rate of 1e-5 and weight decay of 0.1. We set $N_{\min} = 10$ and $N_{\max} = 30$ for all models.

For the right part of Figure \ref{fig:ex}, the Transformer and Mamba setups match the aforementioned configurations for TF and Mamba.
Moreover, we set $[N_{\min}, N_{\max}]$ to $[5, 10]$, $[10, 20]$, $[20, 30]$ and $[30, 40]$ for sequence lengths of 10, 20, 30 and 40 respectively.

For the CoT task experiments, we mainly follow the setup of \citet{hedi_CoT, hedi_linear_attention}.
For the LIS task, we investigate different task lengths $L = \{10, 30, 50, 70$ which denotes the length of the input sequence to solve.
While for the Arithmetic task, we select the task length as $L = \{4, 5, 6, 7\}$.
Here, the task length refers to the number of steps required to incrementally compute the arithmetic expression.
An example when $L=4$ is: $4 \times (8 - 6 / 3) = 4 \times (8 - 2) = 4 \times 6 = 24$.
The model dimensions is selected from \( d = \{32, 64, 128, 256\} \) and the number of layers is set to 3 by default for Transformers.
While for Mamba, under each setting, we adjust the number of layers to match or slightly exceed the size of the Transformer within the same $d$, with roughly two Mamba layers corresponding to one Transformer layer.
All models are trained for 300 epochs using AdamW\cite{AdamW}  with a learning rate of 1e-4 and weight decay of 0.01.
For the results shown in the Figures~\ref{fig:2}, \ref{app:fig:LIS_heatmap} and \ref{app:fig:Arithmetic_heatmap}, we report the average test accuracy over the last five epochs as the final accuracy.
We run all experiments three times and reported the average results.

Furthermore, our experiments were conducted on four 24GB NVIDIA GeForce RTX 3090 GPUs and were completed within five days.
Additionally, more results on the LIS task are provided in Figures \ref{app:fig:LIS_training} and \ref{app:fig:LIS_heatmap}, while additional results on arithmetic tasks are presented in Figures \ref{app:fig:Arithmetic_training} and \ref{app:fig:Arithmetic_heatmap}.

\subsection{More experiment results}

\begin{figure*}[!htbp]
	\centering
	\begin{subfigure}[t]{0.22\linewidth}
		\centering
		\includegraphics[scale=.20]{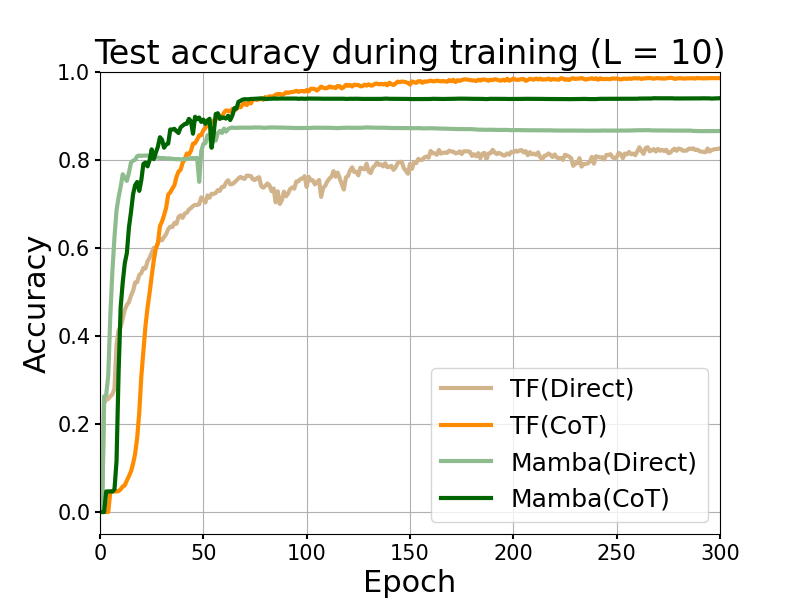}
	\end{subfigure}
	\hspace{0.2cm}
	\begin{subfigure}[t]{0.22\linewidth}
		\centering
		\includegraphics[scale=.20]{pics/figure/LIS/LIS_training_accuracy_L30.png}
	\end{subfigure}
	\hspace{0.2cm}
	\begin{subfigure}[t]{0.22\linewidth}
		\centering
		\includegraphics[scale=.200]{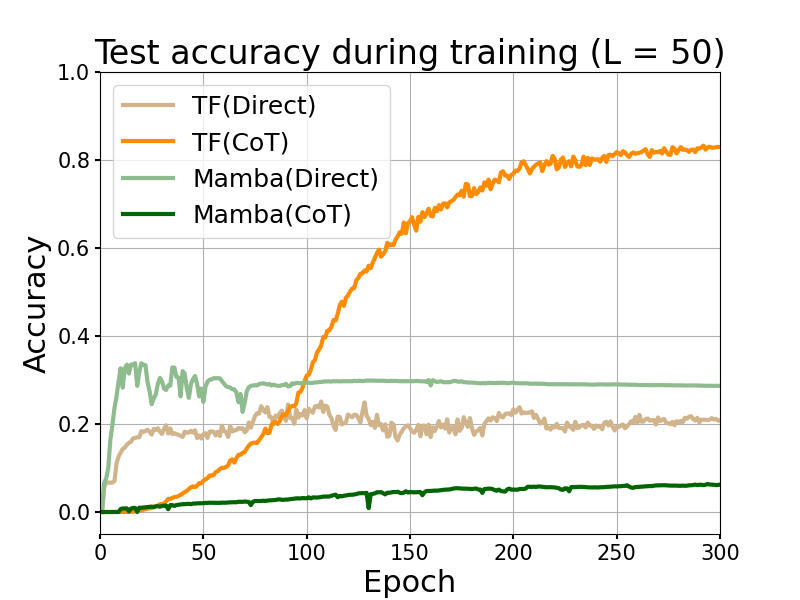}
	\end{subfigure}
	\hspace{0.2cm}
	\begin{subfigure}[t]{0.22\linewidth}
		\centering
		\includegraphics[scale=.200]{pics/figure/LIS/LIS_training_accuracy_L70.png}
	\end{subfigure}
	\vspace{-0cm}
	\caption{
		Test accuracy on LIS tasks during training when the task length $L = 10, 30, 50, 70$ and $d = 256$ (TF denotes Transformer)
	}
	\label{app:fig:LIS_training}
	\vspace*{-0cm}
\end{figure*}

\begin{figure*}[!htbp]
	\centering
	\begin{subfigure}[t]{0.22\linewidth}
		\centering
		\includegraphics[scale=.20]{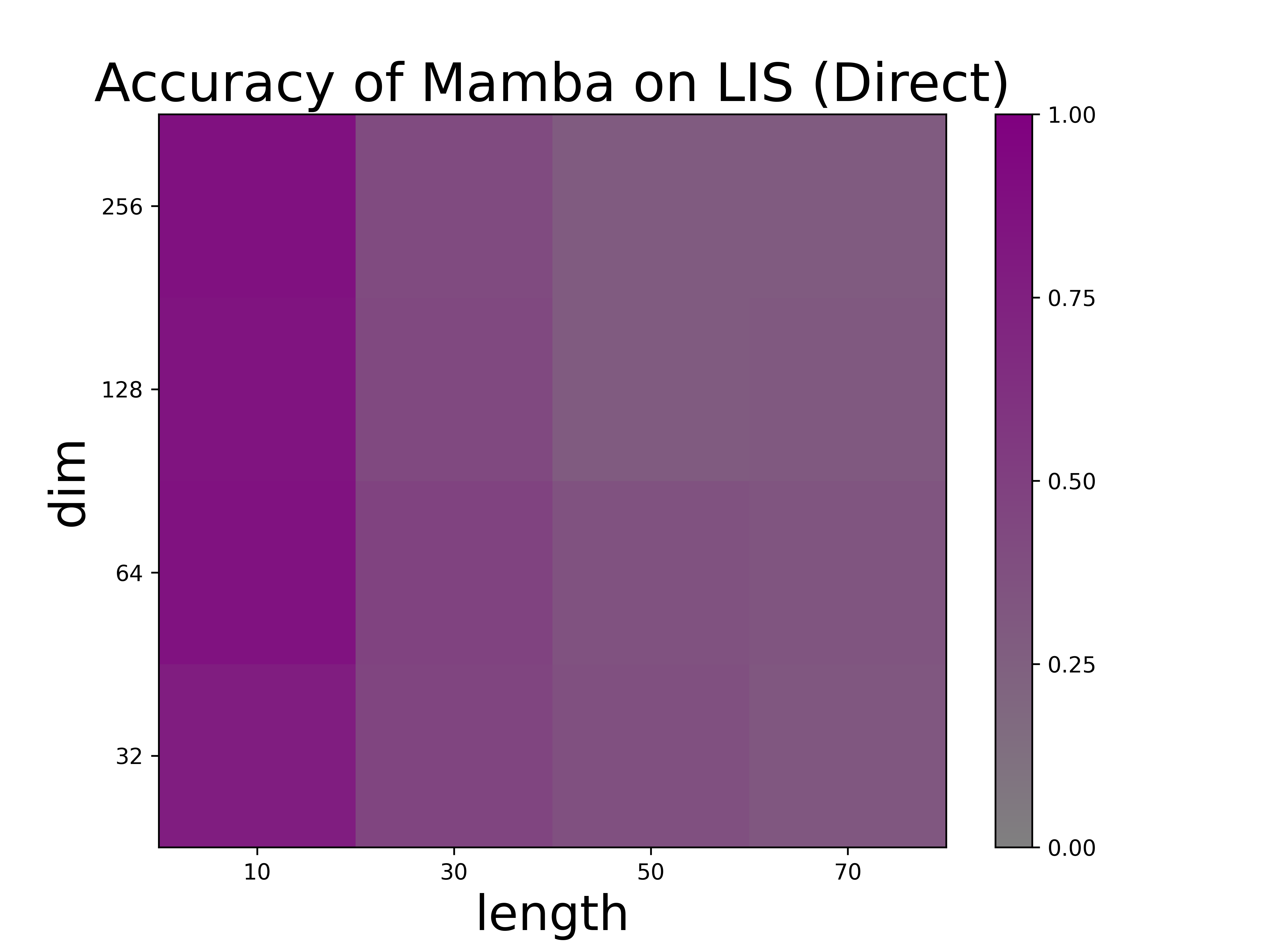}
	\end{subfigure}
	\hspace{0.2cm}
	\begin{subfigure}[t]{0.22\linewidth}
		\centering
		\includegraphics[scale=.20]{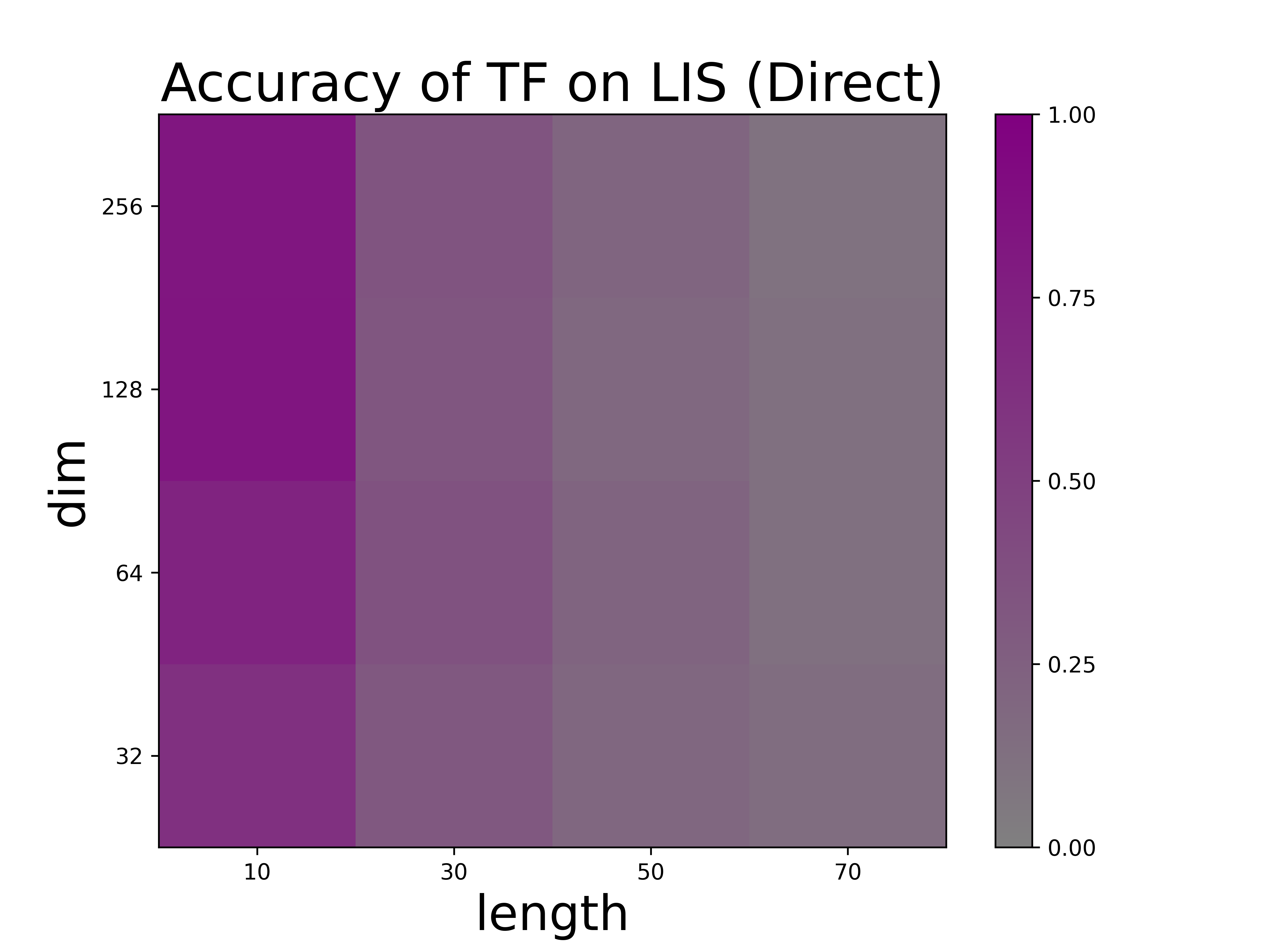}
	\end{subfigure}
	\hspace{0.2cm}
	\begin{subfigure}[t]{0.22\linewidth}
		\centering
		\includegraphics[scale=.200]{pics/figure/LIS/LIS_Mamba_CoT_heatmap.png}
	\end{subfigure}
	\hspace{0.2cm}
	\begin{subfigure}[t]{0.22\linewidth}
		\centering
		\includegraphics[scale=.200]{pics/figure/LIS/LIS_GPT_CoT_heatmap.png}
	\end{subfigure}
	\vspace{-0cm}
	\caption{
		Test accuracy on LIS tasks of Mamba and Transformer across different task lengths and model sizes (with/without CoT).)
	}
	\label{app:fig:LIS_heatmap}
	\vspace*{-0cm}
\end{figure*}

\begin{figure*}[!htbp]
	\centering
	\begin{subfigure}[t]{0.22\linewidth}
		\centering
		\includegraphics[scale=.20]{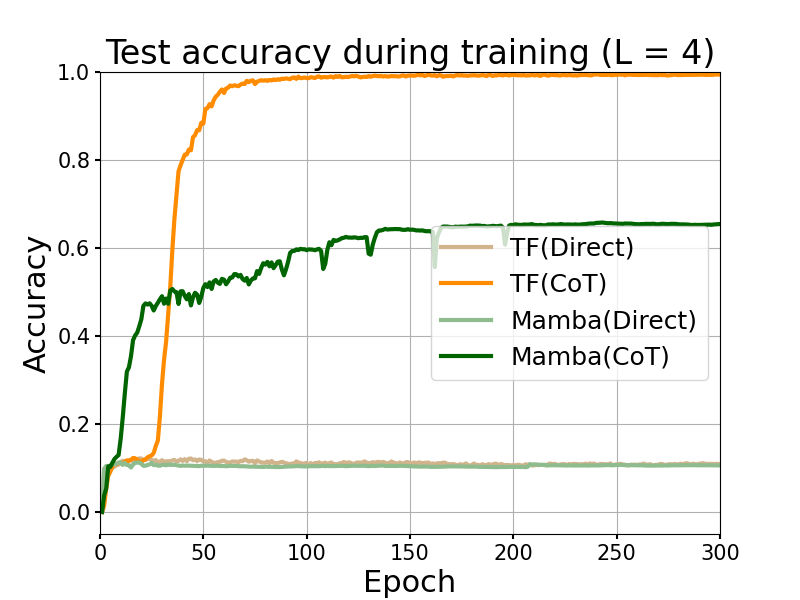}
	\end{subfigure}
	\hspace{0.2cm}
	\begin{subfigure}[t]{0.22\linewidth}
		\centering
		\includegraphics[scale=.20]{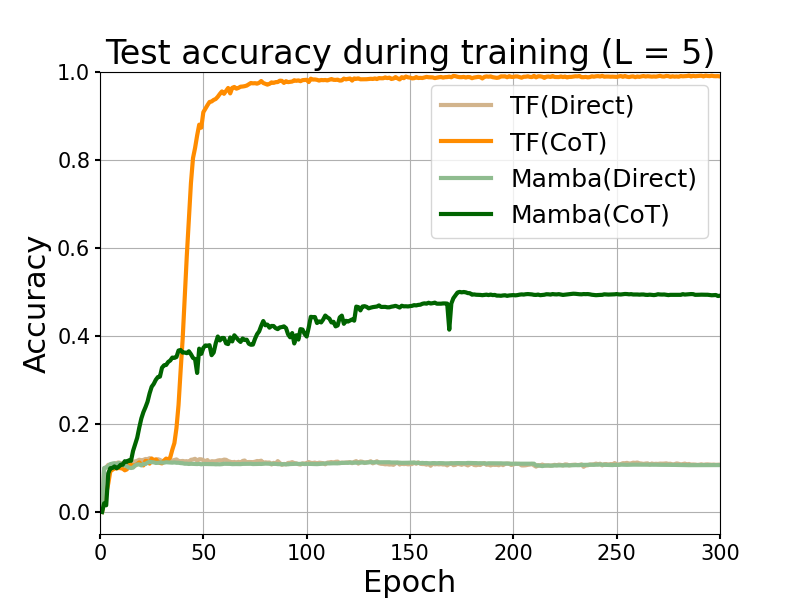}
	\end{subfigure}
	\hspace{0.2cm}
	\begin{subfigure}[t]{0.22\linewidth}
		\centering
		\includegraphics[scale=.200]{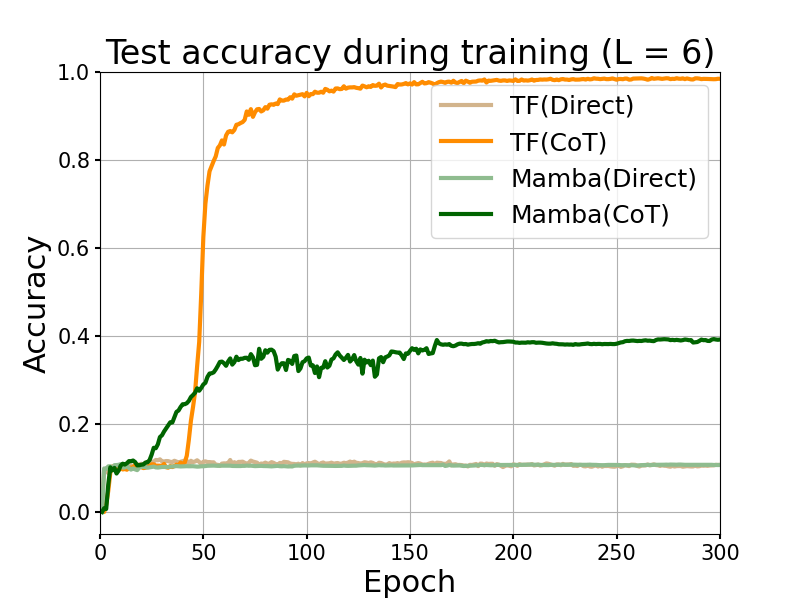}
	\end{subfigure}
	\hspace{0.2cm}
	\begin{subfigure}[t]{0.22\linewidth}
		\centering
		\includegraphics[scale=.200]{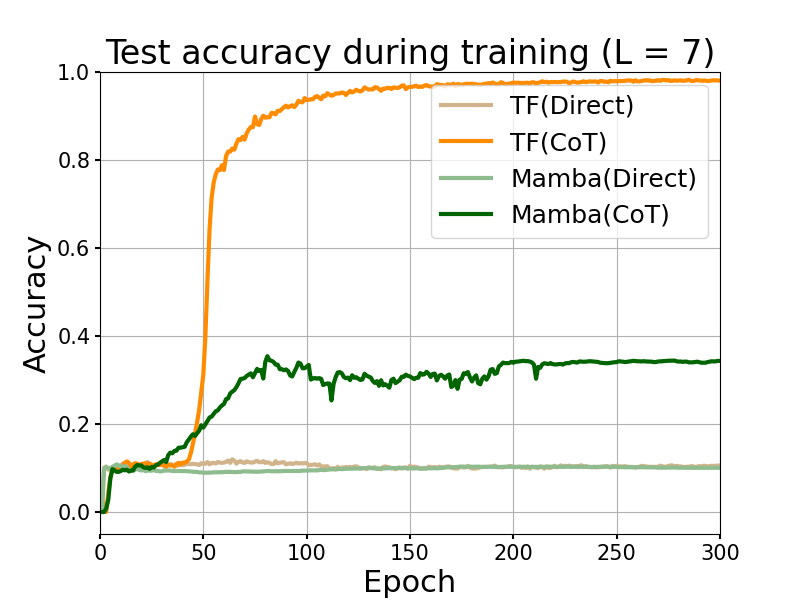}
	\end{subfigure}
	\vspace{-0cm}
	\caption{
		Test accuracy on Arithmetic tasks during training when the task length $L = 4, 5, 6, 7$ and $d = 256$ (TF denotes Transformer). Without CoT, the model’s accuracy is close to that of random guessing.
	}
	\label{app:fig:Arithmetic_training}
	\vspace*{-0cm}
\end{figure*}

\begin{figure*}[!htbp]
	\centering
	\begin{subfigure}[t]{0.22\linewidth}
		\centering
		\includegraphics[scale=.20]{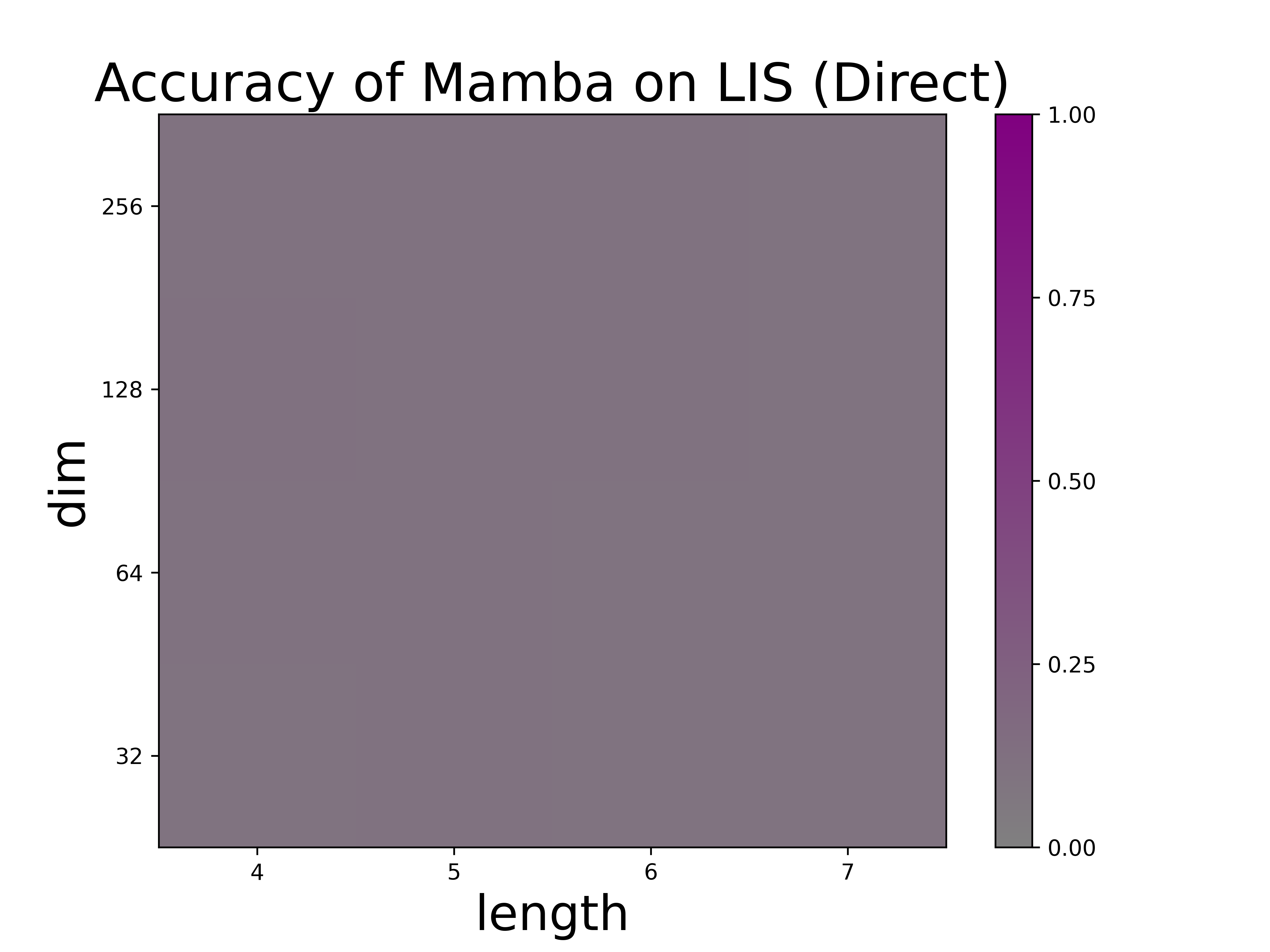}
	\end{subfigure}
	\hspace{0.2cm}
	\begin{subfigure}[t]{0.22\linewidth}
		\centering
		\includegraphics[scale=.20]{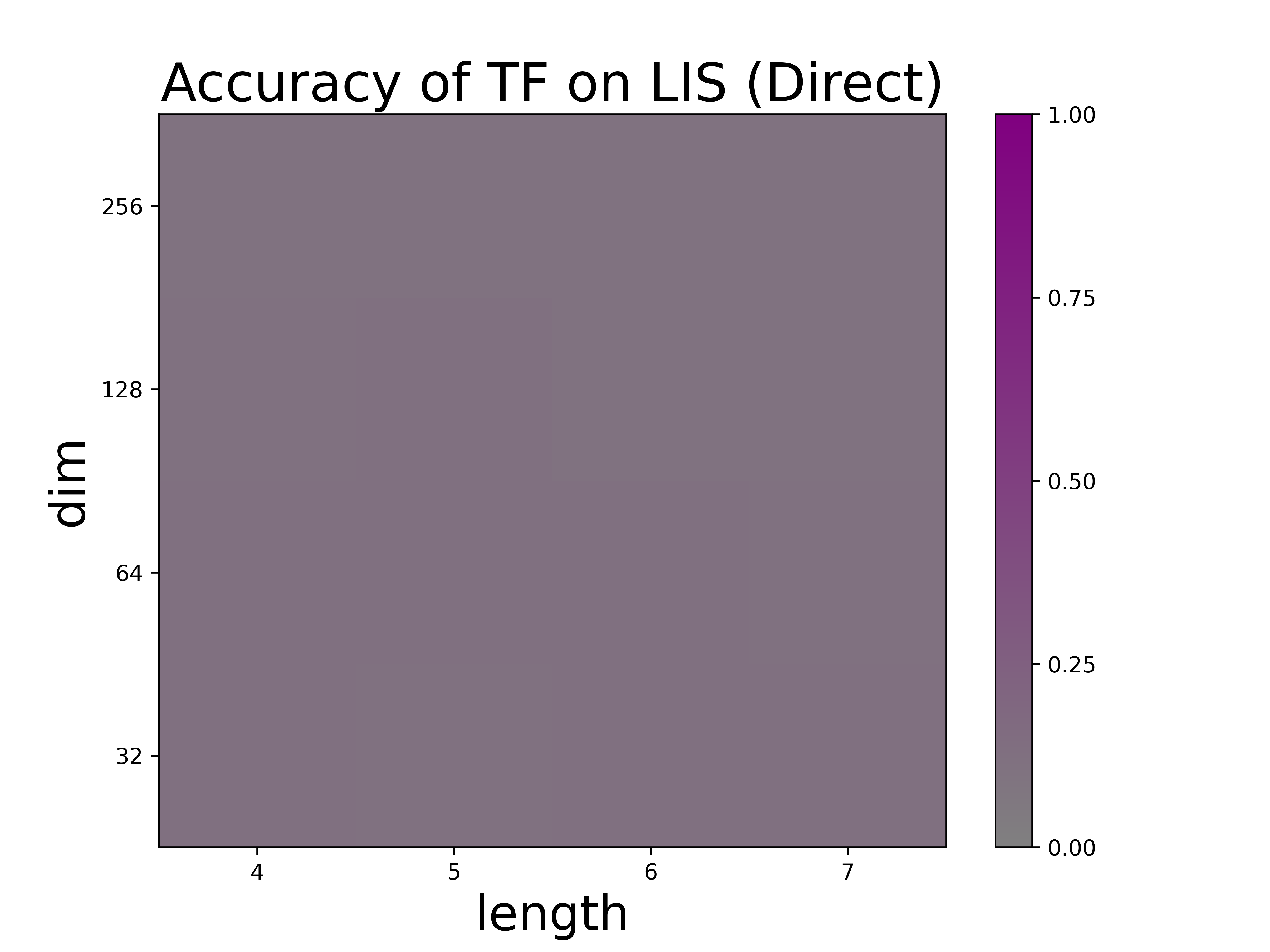}
	\end{subfigure}
	\hspace{0.2cm}
	\begin{subfigure}[t]{0.22\linewidth}
		\centering
		\includegraphics[scale=.200]{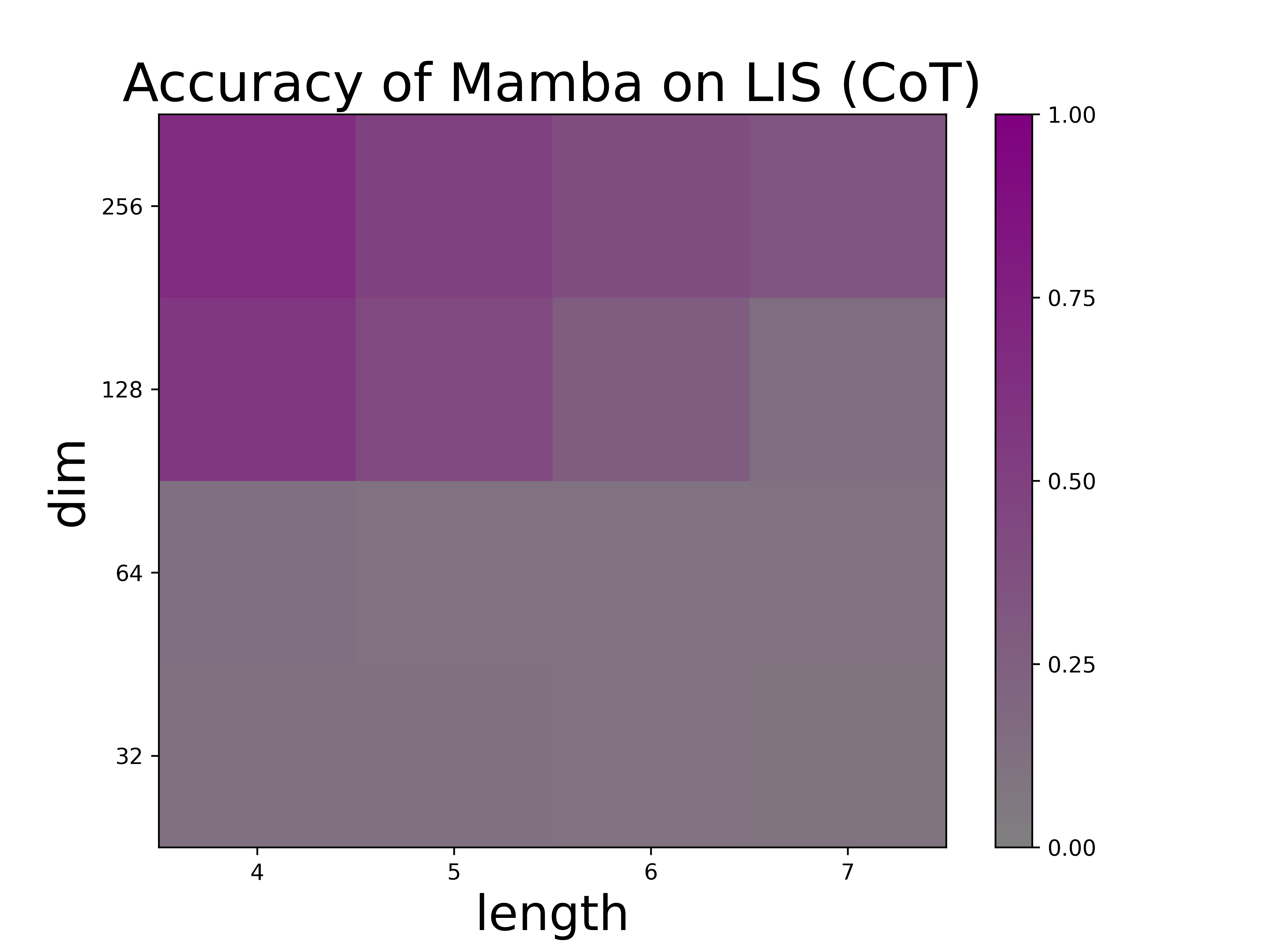}
	\end{subfigure}
	\hspace{0.2cm}
	\begin{subfigure}[t]{0.22\linewidth}
		\centering
		\includegraphics[scale=.200]{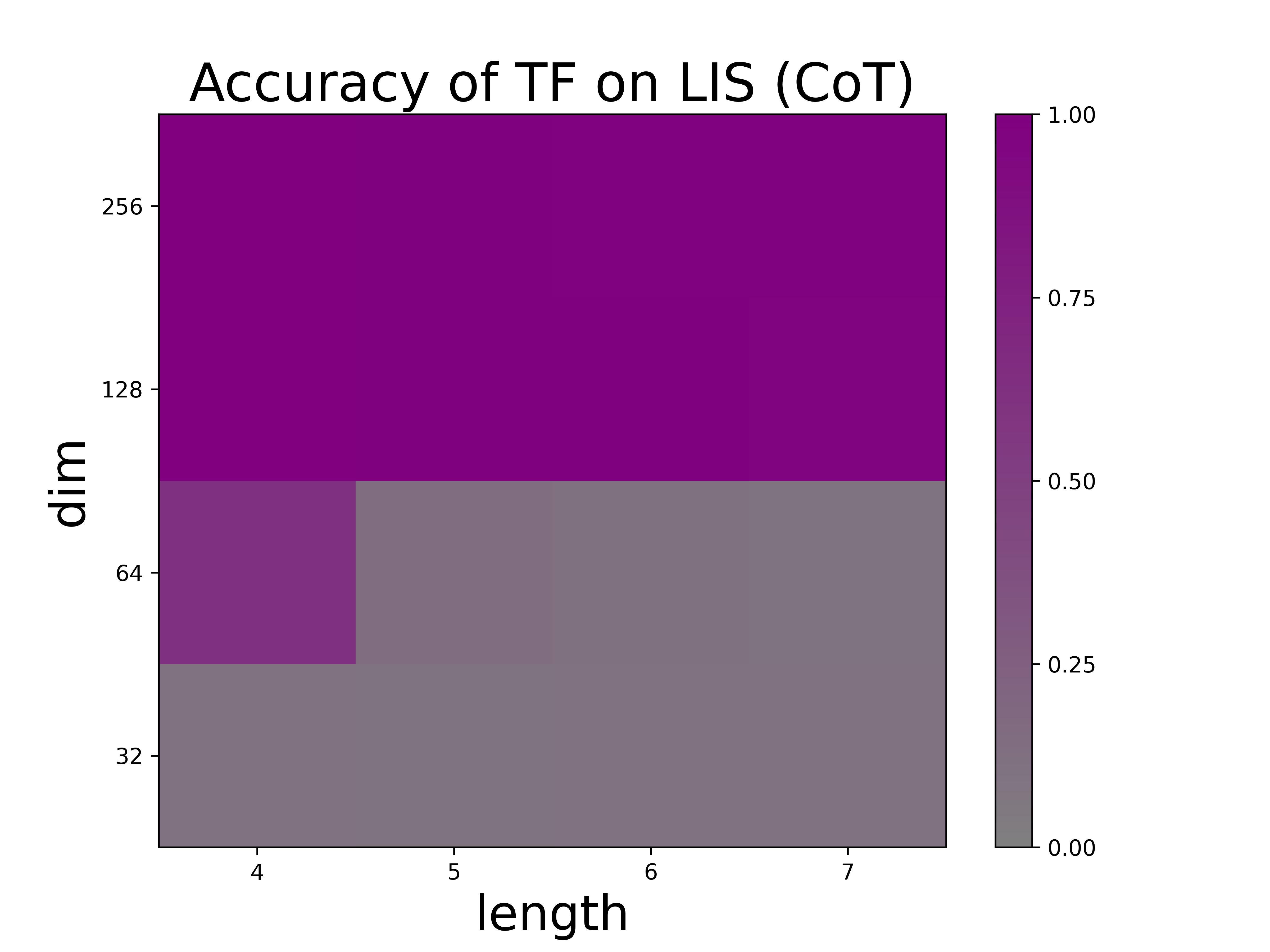}
	\end{subfigure}
	\vspace{-0cm}
	\caption{
		Test accuracy on Arithmetic tasks of Mamba and Transformer across different task lengths and model sizes (with/without CoT).)
	}
	\label{app:fig:Arithmetic_heatmap}
	\vspace*{-0cm}
\end{figure*}

\end{document}